\documentclass[11pt]{article}
\usepackage{fullpage}
\usepackage{enumitem}

\usepackage[dvipsnames]{xcolor}

\usepackage{appendix}
\usepackage{epsf}

\usepackage[dvipsnames]{xcolor}
\usepackage{graphics}
\usepackage[normalem]{ulem}
\usepackage[style = alphabetic,citestyle=alphabetic,minalphanames=4,maxalphanames=4,maxbibnames=99,backend=biber,natbib=true,backref=true]{biblatex}
\DefineBibliographyStrings{english}{%
  backrefpage = {Cited on page},
  backrefpages = {Cited on pages},
}
\usepackage{nicefrac}
\usepackage{subcaption}
\usepackage{psfrag}
\usepackage{comment}
\usepackage{color}
\usepackage{wrapfig}
\usepackage{pgfplots}
\usepackage{pgfplotstable}
\usepackage{multirow}
\usepackage{mathtools}
\usepackage{amsfonts}
\usepackage{amsthm}
\usepackage{amsmath}
\usepackage{amssymb}
\usepackage{scalerel}
\usepackage{nccmath}
\usepackage{tikz}
\usepackage{algorithmic}
\usepackage[ruled,vlined]{algorithm2e}
\usepackage{hyperref}

\hypersetup{colorlinks,linkcolor = violet, urlcolor  = YellowOrange, citecolor = Aquamarine, anchorcolor = YellowOrange}

\DeclareMathOperator*{\pinf}{\vphantom{p}inf}
\DeclareMathOperator*{\psup}{\vphantom{p}sup}
\DeclareMathOperator*{\pmax}{\vphantom{p}max}
\DeclareMathOperator*{\pmin}{\vphantom{p}min}

\newcommand{\real}{\ensuremath{\mathbb{R}}}
\newcommand{\En}{\ensuremath{\mathbb{E}}}

\newcommand{\ind}{\mathbb{I}}
\newcommand{\prob}{\mathbb{P}}
\newcommand{\simplex}{\Delta}

\newcommand{\defn}{:\,=}

\newtheorem{theorem}{Theorem}
\newtheorem{lemma}{Lemma}

\newtheorem{corollary}{Corollary}
\newtheorem{proposition}{Proposition}
\newtheorem{remark}{Remark}
\newtheorem{definition}{Definition}
\newtheorem{assumption}{Assumption}

\newcommand{\1}{\ensuremath{{\sf (i)}}}
\newcommand{\2}{\ensuremath{{\sf (ii)}}}
\newcommand{\3}{\ensuremath{{\sf (iii)}}}

\newcommand\inner[2]{\langle #1, #2 \rangle} 
\newcommand{\ir}{i_1}
\newcommand{\ic}{i_2}
\newcommand{\jj}{j}
\newcommand{\ii}{i}

\newcommand{\lipv}{L_{\textsf{v}}}
\newcommand{\alplb}{\alpha}
\newcommand{\betlb}{\beta}
\newcommand{\prefab}{\pref_{\alplb, \betlb}}
\newcommand{\ld}{q}
\newcommand{\alpz}{\alplb_0}
\newcommand{\betz}{\betlb_0}
\newcommand{\prefz}{\pref_{\alpz, \betz}}
\newcommand{\half}[1]{\left[\frac{1}{2}\right]^{#1}}
\newcommand{\mixd}{\delta}
\newcommand{\mixdb}{\bar{\mixd}}
\newcommand{\numh}{h}
\newcommand{\slope}{m}
\newcommand{\normal}{\mathcal{N}}
\newcommand{\sig}{\sigma}

\newcommand{\pref}{\mathbf{P}}
\newcommand{\pnash}{\mathbf{A}}
\newcommand{\hpnash}{\hat{\pnash}}
\newcommand{\preftil}{\widetilde{\mathbf{P}}}
\newcommand{\pvec}{\mathbf{P}}
\newcommand{\prefm}[1]{\pref_{\textsf{wt}, #1}}
\newcommand{\hpref}{\widehat{\pref}}
\newcommand{\Sc}{S}
\newcommand{\norm}{\|\cdot\|}
\newcommand{\DeltaP}{\Delta_{\pref}}
\newcommand{\DeltaA}{\Delta_{\pnash}}
\newcommand{\cdim}{k}
\newcommand{\adim}{d}
\newcommand{\Ber}{{\sf Ber}}
\newcommand{\ubox}{\mathbb{B}}
\newcommand{\prefeq}{\succeq}
\newcommand{\distf}{\rho}
\newcommand{\rvec}{r}

\newcommand{\oracleS}{\mathcal{O}_{\Sc}}
\newcommand{\robj}{\pi}
\newcommand{\pihat}{\widehat{\pi}}
\newcommand{\pihatplug}{\pihat_{{\sf plug}}}

\newcommand{\samp}{n}

\newcommand{\val}{\mathcal{V}}
\newcommand{\G}{G}

\newcommand{\z}{z}
\newcommand{\wght}{w}
\newcommand{\X}{\mathcal{X}}
\newcommand{\Y}{\mathcal{Y}}

\newcommand{\rew}{r}
\newcommand{\x}{x}
\newcommand{\y}{y}
\newcommand{\alg}{\mathcal{A}}
\newcommand{\robjs}{\robj^*}
\newcommand{\vg}{v}
\newcommand{\f}{f}
\newcommand{\ones}{1}
\newcommand{\zeros}{0}
\newcommand{\ti}{t}
\newcommand{\reg}{r}
\newcommand{\smrad}{\delta}
\newcommand{\smvec}{u}
\newcommand{\subg}{\hat{g}}
\newcommand{\T}{T}
\newcommand{\step}{\eta}
\newcommand{\eps}{\epsilon}
\newcommand{\parop}{\theta}
\newcommand{\sg}{g}
\newcommand{\subd}{\partial}
\newcommand{\maxset}{\Gamma}
\newcommand{\proj}{\Pi}

\newcommand{\p}{p}

\newcommand{\inds}{I}
\newcommand{\hsamp}{\hat{\samp}}

\newcommand{\conf}{\delta}
\newcommand{\prefS}{\mathcal{P}}

\newcommand{\minmax}{\mathfrak{M}}

\newcommand{\const}{c}

\newcommand{\prefbl}{\pref_{\textsf{cr}}}
\newcommand{\wtprefbl}{\widetilde{\prefbl}}
\newcommand{\gap}{\gamma}
\newcommand{\prefhalf}{\pref_{1/2}}
\renewcommand{\prob}{\mathbb{P}}
\newcommand{\hyper}{H}
\newcommand{\dhyp}{\zeta}
\newcommand{\dface}{k_f}

\newcommand{\cmat}{C}
\newcommand{\simp}{\sf{sim}}
\newcommand{\csim}{c_{\simp}}
\newcommand{\cext}{\cmat_{\sf ext}}
\newcommand{\cextj}[1]{\cmat_{\mathsf{ext}, #1}}
\newcommand{\vars}{x}
\newcommand{\cmathat}{\hat{\cmat}}
\newcommand{\cexthat}{\hat{\cmat}_{\sf ext}}
\newcommand{\conset}{J}
\newcommand{\consethat}{\hat{\conset}}
\newcommand{\bvec}[1]{b_{#1}}
\newcommand{\pinv}{\dagger}
\newcommand{\Aalt}{\Phi}

\newcommand{\Znoise}{Z}
\newcommand{\Us}{U}
\newcommand{\Vs}{V}
\newcommand{\Sigs}{\Sigma}
\newcommand{\Ztnoise}{\tilde{\Znoise}}

\makeatletter
\long\def\@makecaption#1#2{
        \vskip 0.8ex
        \setbox\@tempboxa\hbox{\small {\bf #1:} #2}
        \parindent 1.5em  
        \dimen0=\hsize
        \advance\dimen0 by -3em
        \ifdim \wd\@tempboxa >\dimen0
                \hbox to \hsize{
                        \parindent 0em
                        \hfil
                        \parbox{\dimen0}{\def\baselinestretch{0.96}\small
                                {\bf #1.} #2
                                }
                        \hfil}
        \else \hbox to \hsize{\hfil \box\@tempboxa \hfil}
        \fi
        }
\makeatother

\newcommand{\prefov}{{\pref}_{\textsf{ov}}}

\newcommand{\rvy}{y}
\newcommand{\prefex}{\pref_{\textsf{ex}}}

\begin{document}

\begin{center}
{\bf {\LARGE{Preference learning along multiple criteria: \\A game-theoretic perspective}}}

\vspace*{.2in}

\large{
\begin{tabular}{cc}
Kush Bhatia$^\dagger$ & Ashwin Pananjady$^\circ$
\end{tabular}
\begin{tabular}{ccc}
Peter L. Bartlett$^{\dagger, \ddagger}$ & Anca D. Dragan$^{\dagger}$ & Martin J. Wainwright$^{\dagger, \ddagger}$
\end{tabular}
}
\vspace*{.2in}

\begin{tabular}{c}
$^\dagger$Department of Electrical Engineering and Computer Sciences, UC
Berkeley \\ $^\ddagger$Department of Statistics, UC Berkeley
\\ $^\circ$Schools of Industrial \& Systems Engineering and Electrical \& Computer Engineering, \\
 Georgia Tech
\end{tabular}

\vspace*{.2in}

\today

\vspace*{.2in}

\end{center}

\begin{abstract}
The literature on ranking from ordinal data is vast, and there are
several ways to aggregate overall preferences from pairwise
comparisons between objects. In particular, it is well known that any
Nash equilibrium of the zero-sum game induced by the preference matrix
defines a natural solution concept (winning distribution over objects)
known as a von Neumann winner. Many real-world problems, however, are
inevitably multi-criteria, with different pairwise preferences
governing the different criteria. In this work, we generalize the
notion of a von Neumann winner to the multi-criteria setting by taking
inspiration from Blackwell’s approachability. Our framework allows for
non-linear aggregation of preferences across criteria, and generalizes
the linearization-based approach from multi-objective optimization.

From a theoretical standpoint, we show that the Blackwell winner of a
multi-criteria problem instance can be computed as the solution to a
convex optimization problem. Furthermore, given random samples of
pairwise comparisons, we show that a simple ``plug-in'' estimator
achieves near-optimal minimax sample complexity. Finally, we
showcase the practical utility of our framework in a user study on
autonomous driving, where we find that the Blackwell winner
outperforms the von Neumann winner for the overall preferences.
\end{abstract}

\section{Introduction}
\label{sec:intro}

Economists, social scientists, engineers, and computer scientists have
long studied models for human preferences, under the broad umbrella of
social choice theory~\cite{black1948, arrow1951}. Learning from human
preferences has found applications in interactive robotics for
learning reward functions~\cite{sadigh2017, palan2019}, in medical
domains for personalizing assistive devices~\cite{zhang2017,
  biyik2020}, and in recommender systems for optimizing search
engines~\cite{chapelle2012, hofmann2011}. The recent focus on safety
in AI has popularized human-in-the-loop learning methods that use
human preferences in order to promote value
alignment~\cite{christiano2017, saunders2018, amershi2014}.

The most popular form of preference elicitation is to make pairwise
comparisons~\cite{thurstone1927, bradley1952, luce1959}. Eliciting
such feedback involves showing users a pair of objects and asking them
a query: Do you prefer object A or object B? Depending on the
application, an object could correspond to a product in a search
query, or a policy or reward function in reinforcement learning. A
vast body of classical work dating back to Condorcet and
Borda~\cite{condorcet1785, borda1784} has focused on defining and
producing a ``winning'' object from the result of a set of pairwise
comparisons.

Dudik et al.~\cite{dudik2015} proposed the concept of a von Neumann
winner, corresponding to a distribution over objects that beats or
ties every other object in the collection. They showed that under an
expected utility assumption, such a randomized winner always exists
and overcomes limitations of existing winning concepts---the Condorcet
winner does not always exist, while the Borda winner fails an
independence of clones test~\cite{schulze2011}. However, the
assumption of expected utility relies on a strong hypothesis about how
humans evaluate distributions over objects: it posits that the
probability with which any distribution over objects $\robj$ beats an
object is linear in $\robj$.

\begin{figure}[t!]
  \centering\hspace*{-4ex} \captionsetup{font=small}
\begin{tabular}{cc}
  \includegraphics[width=.50\textwidth]{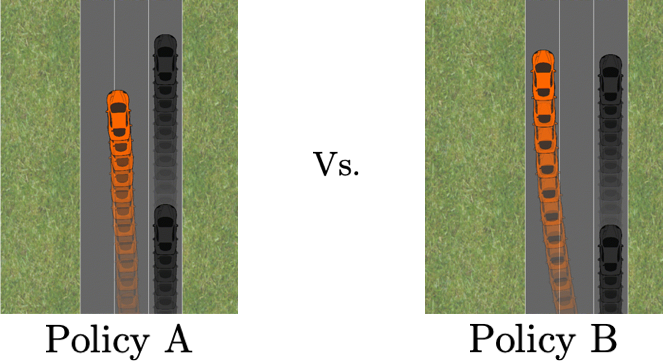}&
  \hspace{3ex}
  \includegraphics[width=.32\textwidth]{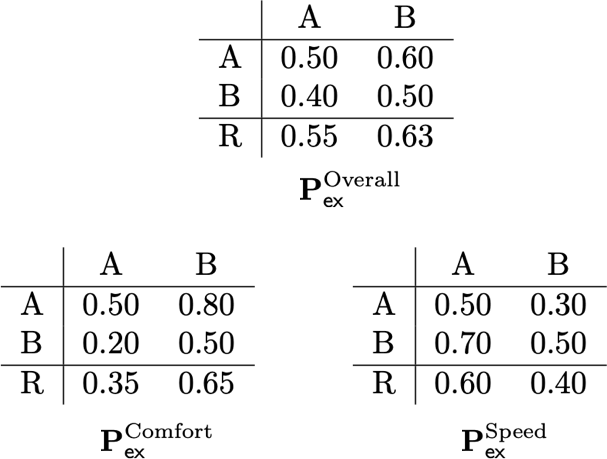}\\ (a)&(b)
\end{tabular}
\caption{\small{(a) Policy A focuses on optimizing comfort, whereas
    policy B focuses on speed, and we consider pairwise comparisons of
    these two policies in different environments.  (b) Preference
    matrices, where entry $(i, j)$ of the matrix contains the
    proportion of comparisons between the pair $(i, j)$ that are won
    by object $i$. (The diagonals are set to half by convention). The
    overall pairwise comparisons are given by the matrix
    $\prefex^\textsf{Overall}$, and preferences along each of the
    criteria by matrices $\prefex^\textsf{Comfort}$ and
    $\prefex^\textsf{Speed}$. Policy R is a randomized policy
    \mbox{$\nicefrac{1}{2}$ A $+ \nicefrac{1}{2}$ B}. While the
    preference matrices satisfy the linearity assumption individually
    along speed and comfort, the assumption is violated overall,
    wherein R is preferred over both A and B.}}
  \label{fig:intro}
\end{figure}

\paragraph{Consequences of assuming linearity:}

In order to better appreciate these consequences, consider as an
example\footnote{Note that while this is an illustrative example, we observe a similar trend in our actual user study in Section~\ref{sec:pol_drive}.} the task of deciding between two policies (say A and B) to
deploy in an autonomous vehicle. Suppose that these policies have been
obtained by optimizing two different objectives, with policy A
optimized for comfort and policy B optimized for
speed. Figure~\ref{fig:intro}(a) shows a snapshot of these two
policies. When compared overall, 60\% of the people preferred Policy A
over B -- making A the von Neumann winner. The linearity assumption
then posits that a randomized policy that mixes between A and B can
\emph{never} be better than both A and B; but we see that the Policy R
= $\nicefrac{1}{2}$ A $+$ $\nicefrac{1}{2}$ B is actually preferred by
a majority over both A and B! Why is the linearity assumption violated
here?

One possible explanation for such a violation is that the comparison
problem is actually \emph{multi-criteria} in nature. If we look at the
preferences for the speed and comfort criteria individually in
Figure~\ref{fig:intro}(b), we see that Policy A does quite poorly on
the speed axis while B lags behind in comfort. In contrast, Policy R
does acceptably well along both the criteria and hence is preferred
overall to both Policies A and B. It is indeed impossible to come to
this conclusion by only observing the overall comparisons. This
observation forms the basis of our main proposal: decompose the single
overall comparison and ask humans to provide preferences along
\emph{simpler} criteria. This decomposition of the comparison task
allows us to place structural assumptions on comparisons along each
criterion. For instance, we may now posit the linearity assumption
along each criterion separately rather than on the overall comparison
task.

In addition to allowing for simplified assumptions, breaking up the
task into such simpler comparisons allows us to obtain richer and more
accurate feedback as compared to the single overall
comparison. Indeed, such a motivation for eliciting simpler feedback
from humans finds its roots in the study of cognitive biases in
decision making, which suggests that the human mind resorts to simple
heuristics when faced with a complicated questions~\cite{tversky1974}.

\paragraph{Contributions:}

In this paper, we formalize these insights and propose a new framework
for preference learning when pairwise comparisons are available along
multiple, possibly conflicting, criteria. As shown by our example in
Figure~\ref{fig:intro}, a single distribution that is the von Neumann
winner along every criteria might not exist.  In order to address this
issue, we formulate the problem of finding the ``best'' randomized
policy by drawing on tools from the literature on vector-valued
pay-offs in game theory. Specifically, we take inspiration from
Blackwell's approachability~\cite{blackwell1956} and introduce the
notion of a Blackwell winner. This solution concept generalizes the
concept of a von Neumann winner, and recovers the latter when there is
only a single criterion present.  Section~\ref{sec:prob} describes
this framework in detail, and Section~\ref{sec:main} collects our
statistical and computational guarantees for learning the Blackwell
winner from data. Section~\ref{sec:pol_drive} describes a user study
with an autonomous driving environment, in which we ask human subjects
to compare self-driving policies along multiple criteria such as
safety, aggressiveness, and conservativeness. Our experiments
demonstrate that the Blackwell winner is able to better trade off
utility along these criteria and produces randomized policies that
outperform the von Neumann winner for the overall preferences.


\section{Related work}
\label{sec:rw}

This paper sits at the intersection of multiple fields of study:
learning from pairwise comparisons, multi-objective optimization,
preference aggregation, and equilibrium concepts in games. Here we
discuss those papers from these areas most relevant to our
contributions.

\paragraph{Winners from pairwise comparisons.}

Most closely related to our work is the field of computational social
choice, which has focused on defining notions of winners from overall
pairwise comparisons (see the survey~\cite{moulin_2016} for a
review). Amongst them, three deterministic notions of a winner---the
Condorcet~\cite{condorcet1785}, Borda~\cite{borda1784}, and
Copeland~\cite{copeland1951} winners---have been widely studied. In
more recent work, Dudik et al.~\cite{dudik2015} introduced
the notion of a (randomized) von Neumann winner.

Starting with the work of Yue et al.~\cite{yue2012}, there have been
several research papers studying an online version of preference
learning, called the Dueling Bandits problem. This is a partial
information version of the classic $K$-armed bandit problem, in which
feedback takes the forms of pairwise comparisons between arms of the
bandit.  Many algorithms have been proposed, including versions that
compete with Condorcet~\cite{zoghi2013, zoghi2015b, ailon2014},
Copeland~\cite{zoghi2015, wu2016}, Borda~\cite{jamieson2015} and von
Neumann~\cite{dudik2015} winners.

\paragraph{Multi-criteria decision making.}

The theoretical foundations of decision making based on multiple
criteria have been widely studied within the operations research
community. This sub-field---called multiple-criteria decision
analysis---has focused largely on scoring, classification, and sorting
based on multiple-criteria feedback. See the
surveys~\cite{pomerol2012multicriterion, zopounidis2002multicriteria}
for thorough overviews of existing methods and their associated
guarantees. The problem of eliciting the user's relative weighting of
the various criteria has also been
considered~\cite{doumpos2007regularized}.  However, relatively less
attention has been paid to the study of randomized decisions and
statistical inference, both of which form the focus of our work.  From
an applied perspective, the combination of multi-criteria assessments
has received attention in disparate fields such as
psychometrics~\cite{papay2011different, mcbee2014combining},
healthcare~\cite{teixeira2008statistical}, and recidivism
prediction~\cite{walters2011taking}. In many of these cases, a variety
of approaches---both linear and non-linear---have been empirically
evaluated~\cite{douglas2010estimating}.  Justification for non-linear
aggregation of scores has a long history in
psychology and the behavioral
sciences~\cite{goldstein1991judgments,frisch1994beyond,tversky1979prospect}.

\paragraph{Blackwell's approachability.}

In the game theory literature, Blackwell~\cite{blackwell1956}
introduced the notion of approachability as a generalization of a
zero-sum game with vector-valued payoffs; see
Appendix~\ref{app:blackwell} for more details.  Blackwell's
approachability and its connections with no-regret learning and
calibrated forecasting have been extensively
studied~\cite{abernethy2011, perchet2013, mannor2014}. These
connections have enabled applications of Blackwell's results to
problems ranging from constrained reinforcement
learning~\cite{miryoosefi2019} to uncertainty estimation for
question-answering tasks~\cite{kuleshov2017}. In contrast with such
applications of the repeated vector-valued game, our framework for
preference learning along multiple criteria deals with a single shot
game and uses the idea of the target set to define the concept of a
Blackwell winner.

\paragraph{Stability of Nash equilibria.}

Another related body of literature focuses on Nash equilibria in games
with perturbed payoffs, under both
robust~\cite{aghassi2006robust,lehrer2012partially} and uncertain or
Bayesian~\cite{fudenberg1993self} formulations; see the recent survey
by Perchet~\cite{perchet2014note}. Perturbation theory for Nash
equilibria has been derived in these contexts, and it is well-known
that the Nash equilibrium is not (in general) stable to perturbations
of the payoff matrix. On the other hand, Dudik et
al.~\cite{dudik2015}, working in the context of dueling bandits,
consider Nash equilibria of perturbed, symmetric, zero-sum games, but
show that the \emph{payoff} of the perturbed Nash equilibrium is
indeed stable. That is, even if the equilibrium itself can change
substantially with a small perturbation of the payoff matrix, the
corresponding payoff is still close to the payoff of the original
equilibrium.  Our work provides a similar characterization for the
multi-criteria setting.

\section{Framework for preference learning along multiple criteria}
\label{sec:prob}

We now set up our framework for preference learning along multiple
criteria.  We consider a collection of $\adim$ objects over which
comparisons can be elicited along $\cdim$ different criteria. We index
the objects by the set $[\adim] \defn \{1, \ldots, \adim\}$ and the
criteria by the set $[\cdim]$.

\subsection{Probabilistic model for comparisons}

Since human responses to comparison queries are typically noisy, we
model the pairwise preferences as random variables drawn from an
underlying population distribution. In particular, the result of a
comparison between a pair of objects $(\ir, \ic)$ along criterion
$\jj$ is modeled as a draw from a Bernoulli distribution, with $p(\ir,
\ic; \jj) = \prob(\ir \prefeq \ic \text{ along criterion } \jj).$ By
symmetry, we must have
\begin{equation}
  \label{eq:symm}
  p(\ic, \ir; \jj) = 1 - p(\ir, \ic; \jj) \quad \mbox{for each triple
    $(\ir, \ic, \jj) \in [\adim] \times [\adim] \times [\cdim]$.}
\end{equation}
Letting $\simplex_d$ denote the $d$-dimensional probability simplex,
consider two probability distributions $\robj_1 , \robj_2 \in
\simplex_{\adim}$ over the $\adim$ objects. With a slight abuse of
notation, let $p(\robj_1, \robj_2; \jj)$ denote the probability with
which an object drawn from distribution $\robj_1$ beats an object
drawn independently from distribution $\robj_2$ along criterion $\jj$. We assume for
each individual criterion $j$ that the probability $p(\robj_1,
\robj_2; \jj)$ is linear in the distributions $\robj_1$ and $\robj_2$,
i.e. that it satisfies the relation
\begin{equation}
  \label{eq:expec_pref}
p(\robj_1, \robj_2; \jj) \defn \En_{ \substack{\ir \sim \robj_1 \\ \ic
    \sim \robj_2}} \left[p(\ir, \ic; \jj) \right].
\end{equation}
Equation~\eqref{eq:expec_pref} encodes the per-criterion linearity
assumption highlighted in Section~\ref{sec:intro}.  We collect the
probabilities $\{ p(\ir, \ic; \jj) \}$ into a \emph{preference tensor}
$\pref \in [0,1]^{\adim \times \adim \times \cdim}$ and denote by
$\prefS_{\adim, \cdim}$ the set of all preference tensors that satisfy
the symmetry condition~\eqref{eq:symm}. Specifically, we have
\begin{equation}
\label{eq:pref_set}
 \prefS_{\adim, \cdim} = \{\pref \in [0,1]^{\adim\times \adim \times
   \cdim}\; | \; \pref(\ir, \ic;\jj) = 1 - \pref(\ic, \ir;\jj) \text{
   for all } (\ir, \ic, \jj) \}\;.
\end{equation}
Let $\pref^{\jj}$ denote the $\adim \times \adim$ matrix corresponding
to the comparisons along criterion $\jj$, so that $p(\robj_1, \robj_2;
\jj) = \robj_1^\top \pref^{\jj} \robj_2$.  Also note that a comparison
between a pair of objects $(\ir,\ic)$ induces a \emph{score vector}
containing $\cdim$ such probabilities. Denote this vector by
$\pvec(\ir, \ic) \in [0, 1]^k$, whose $\jj$-th entry is given by
$p(\ir, \ic; \jj)$. Denote by $\pvec(\robj_1, \robj_2)$ the score
vector for a pair of distribution $(\robj_1, \robj_2)$.

In the single criterion case when $\cdim = 1$, each comparison between
a pair of objects is along an \emph{overall} criterion. We let
$\prefov \in [0,1]^{\adim \times \adim }$ represent such an overall
comparison matrix. As mentioned in Section~\ref{sec:intro}, most
preference learning problems are multi-objective in nature, and the
overall preference matrix $\prefov$ is derived as a non-linear
combination of per-criterion preference matrices $\{ \pref^\jj \}_{j =
  1}^k$. Therefore, even when the linearity
assumption~\eqref{eq:expec_pref} holds across each criterion, it might
not hold for the \emph{overall} preference $\prefov$.

In contrast, when the matrices $\pref^\jj$ are aggregated linearly to
obtain the overall matrix $\prefov$, we recover the assumptions of
Dudik et al.~\cite{dudik2015}.


\subsection{Blackwell winner}
\label{sec:prob-bw}

Given our probabilistic model for pairwise comparisons, we now
describe our notion of a Blackwell winner. When defining a winning
distribution for the multi-criteria case, it would be ideal to find a
distribution $\robjs$ that is a von Neumann winner along \emph{each}
of the criteria separately. 
$p(\pi^*,i;j)\geq 0.5$ for all items $i$ along all criteria $j$.
However, as shown in our example from Figure~\ref{fig:intro}, such a
distribution need not exist: policy A is preferred along the comfort
axis, while policy B along speed.  We thus need a generalization of
the von Neumann winner that explicitly accounts for conflicts between
the criteria.

\begin{figure}[t!]
  \centering\hspace*{-4ex}
  \captionsetup{font=small}
\begin{tabular}{cc}
  \includegraphics[ scale= 0.28]{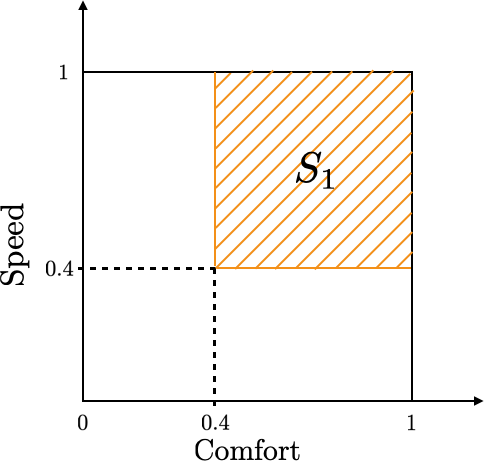}&
  \hspace{4ex}
  \includegraphics[scale = 0.28]{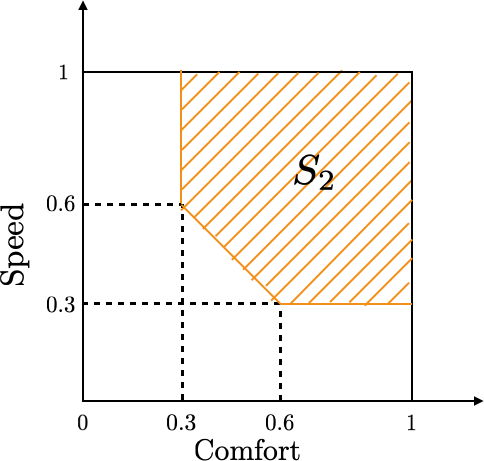}\\
(a)&(b)
\end{tabular}
\caption{\small{ In the context of the example introduced in
    Figure~\ref{fig:intro}, two target sets $\Sc_1$ and $\Sc_2$ that
    capture trade-offs between comfort and speed. Set $\Sc_1$ requires
    feasible score vectors to satisfy 40\% of the population along
    both comfort and speed. Set $\Sc_2$ requires both scores to be
    greater than $0.3$ but with a linear trade-off: the combined score
    must be at least $0.9$.}}
  \label{fig:setup}
\end{figure}

Blackwell~\cite{blackwell1956} asked a related question for the theory of zero-sum games: how can one generalize von Neumann's minimax theorem to vector-valued games? He proposed the notion of a \emph{target set}: a set of acceptable payoff vectors that the first player in a zero-sum game seeks to attain. Within this context, Blackwell proposed the notion of approachability, i.e. how the player might obtain payoffs in a repeated game that are close to the target set on average. We take inspiration from these ideas to define a solution concept for the multi-criteria preference problem.

Our notion of a winner also relies on a target set, which we denote by $\Sc \subset [0,1]^\cdim$, and which in our setting contains \emph{score vectors}. This set provides a way to combine different criteria by specifying combinations of preference scores that are acceptable.
Figure~\ref{fig:setup} provides an example of two such sets.
Observe that for our preference learning problem, the target set $\Sc$ is by definition monotonic with respect to the orthant ordering, that is, if $z_1 \geq z_2$ coordinate-wise, then $\z_2\in \Sc$ implies $\z_1 \in \Sc$.
Our goal is to then produce a distribution $\pi^*$ that can achieve a target score vector for any distribution with which it is compared---that is
$
\pvec(\pi^*, \pi) \in \Sc \text{ for all } \pi \in \simplex_{\adim}.
$
When such a distribution $\pi^*$ exists, we say that the problem instance $(\pref, \Sc)$ is \emph{achievable}.

On the other hand, it is clear that there are problem instances $(\pref, \Sc)$ that are not achievable. While Blackwell's workaround was to move to the setting of repeated games, preference aggregation is usually a one-shot problem. Consequently, our relaxation instead introduces the notion of a \emph{worst-case distance} to the target set. In particular, we measure the distance between any pair of score vectors $u, v \in [0, 1]^k$ as $\rho(u, v) = \| u - v \|$ for some norm $\| \cdot \|$. Using the shorthand $\rho(u, \Sc) \defn \inf_{v \in \Sc} \| u - v \|$, the \emph{Blackwell winner} for an instance $(\pref, \Sc, \norm)$ is now defined as the one that minimizes the maximum distance to the set $\Sc$, i.e.,
\begin{equation} \label{eq:pop-opt}
\pi(\pref, \Sc, \| \cdot \|) \in \arg \min_{\robj \in \simplex_{\adim}}[v(\pi; \pref, \Sc, \norm)], \quad \text{where} \quad
v(\pi; \pref, \Sc, \norm) \defn \pmax_{\pi' \in \simplex_{\adim}}  \distf(\pvec(\pi, \pi'), \Sc) \;.
\end{equation}
Observe that equation~\eqref{eq:pop-opt} has an interpretation as a zero-sum game, where the objective of the minimizing player is to make the score vector $\pref(\robj, \robj')$ as close as possible to the target set $\Sc$.

We now look at commonly studied frameworks for single criterion preference aggregation and multi-objective optimization and show how these can be naturally derived from our framework.

\paragraph{Example: Preference learning along a single criterion.}
A special case of our framework is when we have a single criterion ($\cdim =1$) and the preferences are given by a matrix $\prefov$. 
The score $\prefov(\ir, \ic)$ is a scalar representing the probability with which object $\ir$ beats object $\ic$ in an overall comparison.
As a consequence of the von Neumann minimax theorem, we have
\begin{equation}\label{eq:vn_half}
  \max_{\robj_1 \in \simplex_{\adim}}\min_{\robj_2 \in \simplex_{\adim}} \prefov(\robj_1, \robj_2) {=} \min_{\robj_2 \in \simplex_{\adim}} \max_{\robj_1 \in \simplex_{\adim}} \prefov(\robj_1, \robj_2) {=} \frac{1}{2},
\end{equation}
with any maximizer above called a von Neumann winner~\cite{dudik2015}.
Thus, for \emph{any} preference matrix $\prefov$, a von Neumann winner
is preferred to any other object with probability at least
$\frac{1}{2}$.

Let us show how this uni-criterion formulation can be derived as a
special case of our framework. Consider the target set $\Sc =
[\frac{1}{2}, 1]$ and choose the distance function $\distf(a, b) = |a
- b|$.  By equation~\eqref{eq:vn_half}, the target set $\Sc =
[\frac{1}{2}, 1]$ is achievable \emph{for all} preference matrices
$\prefov$, and so the von Neumann winner and the Blackwell
winner~$\robj(\prefov, [\frac{1}{2}, 1], |\cdot|)$ coincide.
\hfill $\clubsuit$

\paragraph{Example: Weighted combinations of a multi-criterion problem.}

We saw in the previous example that the single criterion preference
learning problem is quite special: achievability can be guaranteed by
the von Neumann winner for set $\Sc = [\frac{1}{2}, 1]$ for any
preference matrix $\prefov$.  One of the common approaches used in
multi-objective optimization to reduce a multi-dimensional problem
to a uni-dimensional counterpart is by introducing a weighted
combinations of objectives.

Formally, consider a weight vector $\wght \in \simplex_{\cdim}$ and
the corresponding preference matrix
\begin{align*}
\pref(\wght) \defn \sum_{\jj \in [\cdim]} \wght_\jj \pref^\jj,
\end{align*}
obtained by combining the preference matrices along the different
criteria. A winning distribution can then be obtained by solving for
the von Neumann winner of $\pref(\wght)$ given by $\robj(\pref(\wght),
[\frac{1}{2}, 1], |\cdot|)$. The following proposition establishes
that such an approach is a special case of our framework,
and conversely, that there are problem instances in our general
framework which cannot be solved by a simple linear weighing of the
criteria.

\begin{proposition}
  \label{prop:lin_uni}
\begin{enumerate}[label=(\alph*)]
\item For every weight vector $\wght\in \simplex_\cdim$, there exists
  a target set $\Sc_\wght \in [0, 1]^k$ such that for any norm
  $\norm$, we have
\begin{equation*}
\pi(\pref, \Sc_\wght, \norm) = \robj(\pref(\wght), [1/2, 1], |\cdot|)
\;\; \text{ for all } \;\; \pref \in \prefS_{\adim, \cdim}.
\end{equation*}
\item Conversely, there exists a set $\Sc$ and a preference tensor
  $\pref$ with a \emph{unique} Blackwell winner $\robjs$ such that for
  all $\wght \in \simplex_\cdim$, exactly one of the following is
  true:
\begin{equation*}
\robj(\pref(\wght), [{1}/{2}, 1], |\cdot|) \neq \robjs \;\; \text{or}
\;\; \arg \max_{\robj \in \simplex_{\adim}} \big \{ \pmin_{\ii \in
  [\adim]} \pref(\robj, \ii) \big \} = \simplex_\adim\;.
\end{equation*}
\end{enumerate}
\end{proposition}
Thus, while the Blackwell winner is always able to recover any linear combination of criteria, the converse is not true. Specifically, part~(b) of the proposition shows that for a choice of preference tensor $\pref$ and target set $\Sc$, either the von Neumann winner for $\pref(w)$ is not equal to the Blackwell winner, or it degenerates to the entire simplex $\simplex_\adim$ and is thus uninformative. Consequently, our framework is strictly more general than weighting the individual criteria.
\hfill $\clubsuit$

\section{Statistical guarantees and computational approaches}\label{sec:main}
In this section, we provide theoretical results on computing the Blackwell winner from samples of pairwise comparisons along the various criteria.

\subsection{Observation model and evaluation metrics.}
We operate in the natural passive observation model, where a sample
consists of a comparison between two randomly chosen objects along a
randomly chosen criterion.  Specifically, we assume access to an
oracle that when queried with a tuple $\eta = (\ir, \ic, \jj)$
comprising a pair of objects $(\ir, \ic)$ and a criterion $\jj$,
returns a comparison \mbox{$\rvy(\eta) \sim \Ber(p(\ir, \ic;
  \jj))$}. Each query to the oracle constitutes one sample. In the
passive sampling model, the tuple of objects and criterion is sampled
uniformly, with replacement, that is $(\ir, \ic) {\sim}
\mathsf{Unif}\{\binom{[\adim]}{2}\} \text{ and } \jj {\sim}
\mathsf{Unif}\{[\cdim]\}$ where $\mathsf{Unif}\{A\}$ denotes the
uniform distribution over the elements of a set $A$.

Given access to samples $\{\rvy_1(\eta_1), \ldots,
\rvy_\samp(\eta_\samp)\}$ from this observation model, we define the
empirical preference tensor (specifically the upper triangular part)
\begin{equation}
  \label{eq:emp_pref}
\hpref_\samp(\ir, \ic, \jj) \defn \frac{\sum_{\ell = 1}^n
  \rvy_\ell(\eta_\ell)\ind[\eta_\ell = (\ir, \ic, \jj)]}{1 \vee
  \sum_{\ell}\ind[\eta_\ell = (\ir, \ic, \jj)]}\quad \text{for } \ir <
\ic \;,
\end{equation}
where each entry of the upper-triangular tensor is estimated using a sample average and the remaining entries are calculated to ensure the symmetry relations implied by the inclusion $\hpref_\samp \in \prefS_{\adim, \cdim}$.

As mentioned before, we are interested in computing the solution
$\robjs \defn \pi(\pref, \Sc, \| \cdot \|)$ to the optimization
problem~\eqref{eq:pop-opt}, but with access only to samples from the
passive observation model. For any estimator $\widehat{\robj} \in
\simplex_\adim$ obtained from these samples, we evaluate its error
based on its value with respect to the tensor $\pref$, i.e.,
\begin{equation} \label{eq:error}
\DeltaP (\widehat{\pi}, \pi) \defn \vg(\widehat{\pi}; S, \pref, \| \cdot \| ) - \vg(\pi^*; S, \pref, \| \cdot \| ).
\end{equation}
Note that the error $\DeltaP$ implicitly also depends on the set $\Sc$ and the norm $\| \cdot \|$, but we have chosen our notation to be explicit only in the preference tensor $\pref$. For the rest of this section, we restrict our attention to convex target sets $\Sc$ and refer to them as \emph{valid sets}.
Having established the background, we are now ready to provide sample complexity bounds on the estimation error $\DeltaP(\widehat{\pi}, \pi^*)$.

\subsection{Upper bounds on the error of the plug-in estimator}

Recall the definition of the function $\vg$ from
equation~\eqref{eq:pop-opt}, and define, for each preference tensor
$\preftil$, an optimizer
\begin{equation} \label{eq:problem}
\pi(\preftil) \in \arg \min_{\pi \in  \simplex_{\adim}} v(\pi; \Sc, \preftil, \norm)\;.
\end{equation}
Also recall the empirical preference tensor $\hpref_\samp$ from equation~\eqref{eq:emp_pref}.
With this notation, the plug-in estimator is given by
$\pihatplug = \pi(\hpref_\samp)$ and the target (or true) distribution by $\pi^* = \pi(\pref)$.

While, our focus in this section is to provide upper bounds on the error of the plug-in estimator $\pihatplug$, we first state a general perturbation bound which relates the error of the optimizer $\robj(\preftil)$ to the deviation of the tensor $\preftil$ from the true tensor $\pref$. We use $\pref(\cdot, i) \in [0,1]^{\adim \times \cdim}$ to denote a matrix formed by viewing the $i$-th slice of $\pref$ along its second dimension. Finally, recall our definition of the error $\DeltaP(\pihat, \pi^*)$ from equation~\eqref{eq:error}.

\begin{theorem}\label{thm:plugin_upper}
Suppose the distance $\distf$ is induced by the norm $\norm_q$ for some $q \geq 1$. Then for each valid target set $\Sc$ and preference tensor $\preftil$, we have
\begin{equation}\label{eq:perturb}
\DeltaP(\pi(\preftil), \pi^*) \leq 2 \max_{\ii \in [\adim]} \|\preftil(\cdot, \ii) - \pref(\cdot, \ii)) \|_{\infty, q}.
\end{equation}
\end{theorem}
Note that this theorem is entirely deterministic: it bounds the
deviation in the optimal solution to the problem~\eqref{eq:pop-opt} as
a function of perturbations to the tensor $\pref$. It also applies
uniformly to all valid target sets $\Sc$. In particular, this result
generalizes the perturbation result of Dudik et
al.~\cite[Lemma~3]{dudik2015} which obtained such a deviation bound
for the single criterion problem with $\robjs$ as the von Neumann
winner. Indeed, one can observe that by setting the distance
$\distf(u, v) = |u-v|$ in Theorem~\ref{thm:plugin_upper} for the
uni-criterion setup, we have the error $\DeltaP(\pi(\preftil), \pi^*)
\leq 2 \|\preftil - \pref\|_{\infty, \infty}$, matching the bound of
\cite{dudik2015}.

Let us now illustrate a consequence of this theorem by specializing it to the plug-in estimator, and with the distances given by the $\ell_\infty$ norm.
\begin{corollary} \label{cor:linfty}
Suppose that the distance $\distf$ is induced by the $\ell_\infty$-norm $\norm_\infty$. Then there exists a universal constant $c>0$ such that given a sample size $\samp >  c\adim^2\cdim \log(\frac{c\adim\cdim}{\conf})$,
we have for each valid target set $\Sc$
\begin{equation} \label{eq:linfty-exp-ub}
\DeltaP(\pihatplug, \pi^*)  \leq \const \sqrt{ \frac{\adim^2 \cdim}{\samp}\log\left(\frac{c \adim \cdim}{\conf} \right)},\;
\end{equation}
with probability greater than $1-\conf$.
\end{corollary}
The bound~\eqref{eq:linfty-exp-ub} implies that the plug-in estimator
$\pihatplug$ is an $\eps$-approximate solution whenever the number of
samples scales as $\samp = \widetilde{
  O}(\frac{\adim^2\cdim}{\eps^2})$. Observe that this sample
complexity scales quadratically in the number of objects $\adim$ and
linearly in the number of criteria $\cdim$. This scaling represents
the effective dimensionality of the problem instance, since the
underlying preference tensor $\pref$ has $O(\adim^2\cdim)$ unknown
parameters.  Notice that the corollary holds for sample size $\samp=
\widetilde{\Omega}(\adim^2\cdim)$; this should not be thought of as
restrictive, since otherwise, the bound~\eqref{eq:linfty-exp-ub} is
vacuous.

\subsection{Information-theoretic lower bounds}
\label{sec:lower}

While Corollary~\ref{cor:linfty} provides an upper bound on the error
of the plug-in estimator that holds \emph{for all} valid target sets
$\Sc$, it is natural to ask if this bound is sharp, i.e., whether
there is indeed a target set $\Sc$ for which one can do no better than
the plug-in estimator.

In this section, we address this question by providing lower bounds on
the minimax risk
\begin{equation}
\minmax_{\samp, \adim, \cdim} (\Sc, \| \cdot \|_\infty) \defn
\pinf_{\pihat} \psup_{\pref \in \prefS} \En \left[ \DeltaP (\pihat,
  \pi^*) \right],
\end{equation}
where the infimum is taken over \emph{all} estimators that can be
computed from $n$ samples from our observation model. It is important
to note that the error $\DeltaP$ is computed using the $\ell_\infty$
norm and for the set $\Sc$. Our lower bound will apply to the
particular choice of target set $S_0 = [1/2, 1]^{k}$.

\begin{theorem}
\label{thm:lower_passive}
There are universal constants $c, c'$ such that for all $\adim \geq
4$, $\cdim \geq 2$, and $\samp \geq c\adim^4\cdim$, we have
\begin{equation}
\label{eq:minimax-lb}
\minmax_{\samp, \adim, \cdim} (S_0, \| \cdot \|_\infty ) \geq \const'
\sqrt{\frac{\adim^2\cdim}{\samp}}.
\end{equation}
\end{theorem}
Comparing equations and~\eqref{eq:linfty-exp-ub}
and~\eqref{eq:minimax-lb}, we see that for the $\ell_\infty$-norm and
the set $S_0$, we have provided upper and lower bounds that match up
to a logarithmic factor in the dimension. Thus, the plug-in estimator
is indeed optimal for this pair $(\| \cdot \|_\infty, S_0)$.  Further,
observe that the above lower bound is non-asymptotic, and holds for
all values of $\samp \gtrsim \adim^4\cdim$.  This condition on the
sample size arises as a consequence of the specific packing set used
for establishing the lower bound, and improving it is an interesting
open problem.

However, Theorem~\ref{thm:lower_passive} raises the question of whether the set $\Sc_0$ is
special, or alternatively, whether one can obtain an $\Sc$-dependent
lower bound. The following proposition shows that at least
\emph{asymptotically}, the sample complexity for \emph{any} polyhedral
set $\Sc$ obeys a similar lower bound.
\begin{proposition}[Informal]\label{prop:lower_gens}
  Suppose that we have a valid polyhedral target set $\Sc$, and that
  $\adim \geq 4$. There exists a positive integer $\samp_0(\adim,
  \cdim, \Sc)$ such that for all $\samp \geq \samp_0(\adim, \cdim,
  \Sc)$ we have
  \begin{equation} \label{eq:minimax-lb-asymptotic}
    \minmax_{\samp, \adim, \cdim} (\Sc, \norm_\infty ) \gtrsim \sqrt{\frac{\adim^2\cdim}{\samp}}\;.
  \end{equation}
\end{proposition}
We defer the formal statement and proof of this proposition to Appendix~\ref{app:proof_main}. This proposition establishes that the plugin estimator $\pihatplug$ is indeed optimal in the $\ell_\infty$ norm for a broad class of sets~$\Sc$. Note that the result is asymptotic in nature: in order for the proposition to hold, we require that the number of samples is greater than the value $\samp_0$. This number $\samp_0$ depends on problem dependent parameters, and we provide an exact expression for $\samp_0$ in the appendix.

\subsection{Instance-specific analysis for the plug-in estimator}
In the previous section we established that the error $\DeltaP(\pihatplug, \robjs)$ of the plug-in estimator scales as $\widetilde{O}\left(\sqrt{\frac{\adim^2\cdim}{\samp}}\right)$ for any choice of preference tensor $\pref$ and target set $\Sc$ when the distance function $\distf = \norm_\infty$. In this section, we study the adaptivity properties of the plug-in estimator $\pihatplug$ and obtain upper bounds on the error $\DeltaP(\pihatplug, \robjs)$ that depend on the properties of the underlying problem instance.

In the main text, we will restrict our focus to the uni-criterion setup with $\cdim = 1$ and the target set $\Sc = [\frac{1}{2},1]$, in which case the Blackwell winner coincides with the von Neumann winner.  Furthermore, we will consider the case where the preference matrix $\pref$ has a unique von Neumann winner $\robjs$. This is formalized in the following assumption.
\begin{assumption}[Unique Nash equilibrium]\label{ass:unique-nash-main}
  The matrix $\pref$ belongs to the set of preference matrices $\prefS_{\adim,1}$ and has a unique mixed Nash equilibrium $\robjs$, that is, $\robjs_i > 0$ for all $i \in [d]$.
\end{assumption}
For the more general analysis, we refer the reader to Appendix~\ref{app:local-asymp}. For any preference matrix $\pref \in \prefS_{\adim, 1}$ and the Bernoulli passive sampling model discussed in Section~\ref{sec:main} let us represent by $\Sigma_i$ the diagonal matrix corresponding to the variances along the $i^{th}$ column of the matrix $\pref$ with
\begin{align*}
  \Sigma_i = \text{diag}(\pref({1,i})\cdot(1-\pref({1,i})), \ldots, \pref({\adim,i})\cdot(1-\pref({\adim,i})).
\end{align*}
Given this notation, we now state an informal corollary (of Theorem~\ref{thm:local-nash} in the appendix) which shows that the error $\DeltaP(\pihatplug, \robjs)$ depends on the worst-case alignment of the Nash equilibrium $\robjs$ with the underlying covariance matrices $\Sigma_i$.

\begin{corollary}[Informal]\label{cor:local-mixed-main}
For any preference matrix $\pref$ satisfying Assumption~\ref{ass:unique-nash-main}, confidence $\delta >0$, and number of samples $\samp > \samp_0(\pref, \delta)$, we have that the error $\DeltaP$ of the plug-in estimate $\pihatplug$ satisfies
\begin{align}
\DeltaP(\pihatplug, \robjs) &\leq c\cdot\sqrt{\frac{\sig_{\pref}^2\adim^2}{\samp}\log\left( \frac{\adim}{\delta}\right)},
\end{align}
with probability at least $1-\delta$, where the variance $\sig_{\pref}^2 \defn \max_{i \in [d]} (\robjs)^{\top}\Sigma_i \robjs$.
\end{corollary}
We defer the proof of the above to Appendix~\ref{app:local-asymp}. A few comments on the above corollary are in order. Observe that it gives a high probability bound on the error $\DeltaP$ of the plug-in estimator $\pihatplug$. Compared with the upper bounds of Corollaries~\ref{cor:linfty} and~\ref{cor:lone}, the asymptotic bound on the error above is instance-dependent -- the effective variance $\sig_{\pref}^2$ depends on the underlying preference matrix $\pref$. In particular, this variance measures how well the underlying von Neumann winner $\robjs$ aligns with the variance associated with each column of the matrix $\pref$. In the worst case, since each entry of $\pref$ is bounded above by $1$, the variance $\sigma_\pref^2 = 1$ and we recover the upper bounds from Corollaries~\ref{cor:linfty} and~\ref{cor:lone} for the uni-criterion case. More interestingly, the bound provided by Corollary~\ref{cor:local-mixed-main} can be significantly sharper (by a possibly dimension-dependent factor) than its worst-case counterpart. We explore concrete examples of this in Appendix~\ref{app:local-asymp}.

\subsection{Computing the plug-in estimator}
In the last few sections, we discussed the statistical properties of the plug-in estimator, and showed that its sample complexity was optimal in a minimax sense. We now turn to the algorithmic question: how can the plug-in estimator $\pihatplug$ be computed? Our main result in this direction is the following theorem that characterizes properties of the objective function $v(\pi; \pref, \Sc, \| \cdot \|)$.

\begin{theorem} \label{thm:opt}
Suppose that the distance function is given by an $\ell_q$ norm $\| \cdot \|_q$ for some $q \geq 1$. Then for each valid target set $\Sc$, the objective function $v(\pi; \pref, \Sc, \| \cdot \|_q)$ is convex in $\pi$, and Lipschitz in the $\ell_1$ norm, i.e.,
\begin{align*}
|v(\pi_1; \pref, \Sc, \| \cdot \|_q) - v(\pi_2; \pref, \Sc, \| \cdot \|_q)| \leq \cdim^{\frac{1}{q}}\cdot  \| \pi_1 - \pi_2 \|_1 \text{ for each } \pi_1, \pi_2 \in \simplex_{\adim}.
\end{align*}
\end{theorem}

Theorem~\ref{thm:opt} establishes that the plug-in estimator can indeed be computed as the solution to a (constrained) convex optimization problem. In Appendix~\ref{app:add_res}, we discuss a few specific algorithms based on zeroth-order and first-order methods for obtaining such a solution and an analysis of the corresponding iteration complexity for these methods; see Propositions~\ref{prop:conv_zero} and~\ref{prop:conv_fopl} in the appendix. These methods differ in the way they access the target set $\Sc$: while zeroth-order methods require a \emph{distance oracle} to the target set, the first-order methods require a stronger \emph{projection oracle} to this constraint set.

\section{Autonomous driving user study}
\label{sec:pol_drive}
In order to evaluate the proposed framework, we applied it to an autonomous driving environment. The objective is to study properties of the randomized policies obtained by our multi-criteria framework---the Blackwell winner for specific choices of the target set---and compare them with the alternative approaches of linear combinations of criteria and the single-criterion (overall) von Neumann winner.
We briefly describe the components of the experiment here; see Appendix~\ref{app:exp} for more details.

\paragraph{Self-driving Environment.} Figure~\ref{fig:intro}(a) shows a snapshot of one of the worlds in this environment with the autonomous car shown in orange.
We construct three different worlds in this environment:
\begin{itemize}
  \item[W1:] The first world comprises an empty stretch of road with no obstacles (20 steps). \vspace{-1mm}
  \item[W2:] The second world consists of cones placed in a given
    sequence (80 steps).\vspace{-1mm}
  \item[W3:] The third world has additional cars driving at varying
    speeds in their fixed lanes (80 steps).\vspace{-1mm}
\end{itemize}

\paragraph{Policies.} For our \emph{base policies}, we design five
different reward functions encoding different self-driving
behaviors. These polices, named Policy A-E, are then set to be the
model predictive control based policies based on these reward
functions wherein we fix the planning horizon to $6$.  See
Appendix~\ref{app:exp} for a detailed description of these reward
functions.

A \emph{randomized policy} $\robj \in \simplex_5$ is given by a
distribution over the base policies A-E. Such a randomized policy is
implemented in our environment by randomly sampling a base policy from
the mixture distribution after every $H = 18$ time steps and executing
this selected policy for that duration. To account for the
randomization, we execute each such policy for $5$ independent runs in
each of the worlds and record these behaviors.

\paragraph{Subjective Criteria.} We selected five subjective criteria
with which to compare the policies, with questions asking which of the
two policies was C1: Less aggressive,\; C2: More predictable,\; C3:
More quick,\; C4:~More conservative,\; and had C5: Less collision
risk. Such a framing of question ensures that higher score value along
any of C1-C5 is preferred; thus a higher score along C1 would imply
less aggressive while along C2 would mean more predictable.

In addition to these base criteria, we also consider an
\emph{Overall Preference} which compares any pair of policies in an
aggregate manner. For this criterion, the users were asked to select
the policy they would prefer when riding to their destination.
Additionally, we also asked the users to rate the importance of each
criterion in their overall preference.

\paragraph{Main Hypotheses.}
Our hypotheses focus on comparing the randomized policies given by the
Blackwell winner, the overall von Neumann winner, and those given by
weighing the criteria linearly.
\begin{itemize}
  \item[MH1] There exists a set $\Sc$ such that the Blackwell winner
    with respect to $\Sc$ and $\ell_\infty$-norm produced by our
    framework outperforms the overall von Neumann winner.
  \item[MH2] The Blackwell winner for oblivious score sets $\Sc$
    outperforms both oblivious\footnote{We use the term oblivious to
    denote variables that were \emph{fixed} before the data collection
    phase and data-driven to denote those which are based on collected
    data.} and data-driven weights for linear combination of criteria.
\end{itemize}

\paragraph{Independent Variables.}
The independent variable of our experiment is the choice of algorithms
for producing the different randomized winners. These comprise the von
Neumann winner based on overall comparisons, Blackwell winners based
on two oblivious target sets, and 9 different linear combinations
weights (3 data-driven and 6 oblivious).

We begin with the two target sets $\Sc_1$ and $\Sc_2$ for our
evaluation of the Blackwell winner, which were selected in a
data-oblivious manner. Set $\Sc_1$ is an axis-aligned set promoting
the use of safer policies with score vector constrained to have a
larger value along the collision risk axis. Similar to
Figure~\ref{fig:setup}(b), the set $\Sc_2$ adds a linear constraint
along aggressiveness and collision risk. This target set thus favors
policies that are less aggressive and have lower collision risk.  For
evaluating hypothesis MH2, we considered several weight vectors, both
oblivious and data-dependent, comprising the average of the users'
self-reported weights, that obtained by regressing the overall
criterion on C1-C5, and a set of oblivious weights. See
Appendix~\ref{app:exp} for details of the sets $\Sc_1$ and $\Sc_2$,
and the weights $\wght_{1:9}$.

\paragraph{Data collection.}

The experiment was conducted in two phases, both of which involved
human subjects on Amazon Mechanical Turk (Mturk).  See
Appendix~\ref{app:exp} for an illustration of the questionnaire.

The first phase of the experiment involved preference elicitation for
the five base policies A-E. Each user was asked to provide comparison
data for all ten combinations of policies. The cumulative comparison
data is given in Appendix~\ref{app:exp}, and the average weight vector
elicited from the users was found to be $\wght_1 = \left[ 0.21, 0.19,
  0.20, 0.18, 0.22\right]$. We ran this study with 50 subjects.

In the overall preference elicitation, we saw an approximate ordering
amongst the base policies: \mbox{$\text{C} \succ \text{E} \succsim
  \text{D} \succsim \text{B} \succ \text{A}$}. Thus, Policy C was the
von Neumann winner along the overall criterion. For each of the linear
combination weights $\wght_1$ through $\wght_9$, Policy C was the
weighted winner. The Blackwell winners R1 and R2 for the sets $\Sc_1$
and $\Sc_2$ with the $\ell_\infty$ distance were found to be
$\text{R1} = [0.09, 0.15, 0.30, 0.15, 0.31]$ and $\text{R2} = [0.01,
  0.01, 0.31, 0.02, 0.65]$.

In the second phase, we obtained preferences from a set of 41 subjects
comparing the randomized polices R1 and R2 with the baseline policies
A-E. The results are aggregated in Table~\ref{tab:phase_one} in
Appendix~\ref{app:exp}.

\paragraph{Analysis for main hypotheses.}

Given that the overall von Neumann winner and those corresponding to
weights $\wght_{1:9}$ were all Policy C, hypotheses MH1 and MH2
reduced to whether users prefer at least one of \{R1, R2\} to the
deterministic policy C, that is whether $\prefov(\text{C}, \text{R1})
< 0.5$ or $\prefov(\text{C}, \text{R2}) < 0.5$.

Policies C and E were preferred to R1 by $0.71$ and $0.61$ fraction of
the respondents, respectively. On the other hand, R2 was preferred to
the von Neumann winner C by $0.66$ fraction of the subjects. Using the
data, we conducted a hypothesis test with the null and alternative
hypotheses given by
\begin{align*}
  H_0:
  \prefov(\text{C}, \text{R2}) \geq 0.5, \quad \text{ and } \quad
  H_1:
  \prefov(\text{C}, \text{R2}) < 0.5.
\end{align*}
Among the hypotheses that make up the (composite) null, our samples
have the highest likelihood for the distribution $\Ber(0.5)$. We
therefore perform a one-sided hypothesis test with the Binomial
distribution with number of samples $n = 41$, success probability $p =
0.5$ and number of successes $x = 14$ (indicating number of subjects
which preferred Policy C to R1). The p-value for this test was
obtained to be $0.0298$.  This supports both our claimed hypotheses
MH1 and MH2.

\section{Discussion and future work}

In this paper, we considered the problem of eliciting and learning
from preferences along multiple criteria, as a way to obtain rich
feedback under weaker assumptions. We introduced the notion of a
Blackwell winner, which generalizes many known winning solution
concepts. We showed that the Blackwell winner was efficiently
computable from samples with a simple and optimal procedure, and also
that it outperformed the von Neumann winner in a user study on
autonomous driving. Our work raises many interesting follow-up
questions: How does the sample complexity vary as a function of the
preference tensor $\pref$? Can the process of choosing a good target
set be automated? What are the analogs of our results in the setting
where pairwise comparisons can be elicited actively?

\section*{Acknowledgments}
We would like to thank Niladri Chatterji, Robert Kleinberg and Karthik
Sridharan for helpful discussions, and Andreea Bobu, Micah Carroll,
Lawrence Chan and Gokul Swamy for helping with the user study setup.

KB is supported by a JP
Morgan AI Fellowship, and AP was supported by a Swiss Re research fellowship at the Simons
Institute for the Theory of Computing. This work was partially supported by NSF grant DMS-2023505 to PLB, by Office of
Naval Research Young Investigator Award, NSF CAREER and a AFOSR grant to ADD, and
by Office of Naval Research Grant DOD ONR-N00014-18-1-2640 to MJW.

\newpage
\appendix
\appendixpage

\section{Blackwell's approachability}
\label{app:blackwell}

Blackwell~\cite{blackwell1956} introduced the concept of
approachability as a generalization of the minimax theorem to
vector-valued payoffs. Formally, a Blackwell game is an extension of
two-player zero-sum games with vector-valued reward functions.

Let $\X, \Y$ denote the action spaces for the two players and $\rew :
\X\times \Y \mapsto \real^\cdim$ be the corresponding vector-valued
reward function. Further, let $\Sc \subseteq \real^\cdim$ denote a
target set. The objective of player $1$ is to ensure that the reward
vector $\rew$ lies in the set $\Sc$ while that of player 2 is ensure
that the reward $\rew$ lies outside this set
$\Sc$. Following~\cite{abernethy2011}, we introduce the notion of
satisfiability and response-satisfiability.
\begin{definition}[Satisfiability]
  For a Blackwell game parameterized by $(\X, \Y, \rew, \Sc)$, we say
  that,
  \begin{itemize}
    \item $\Sc$ is \emph{satisfiable} if there exists $\x \in \X$ such
      that for all $\y \in \Y$, we have that $\rew(\x, \y) \in \Sc$.
    \item $\Sc$ is \emph{response-satisfiable} if for every $\y \in \Y$, there exists $\x \in \X$ such that $\rew(\x, \y) \in \Sc$.
  \end{itemize}
\end{definition}
In the case of scalar rewards, Von Neumann's minimax theorem indicates
that any set which is satisfiable is also response-satisfiable. In
other words, there exists a strategy for Player $1$, oblivious of
Player $2$'s strategy which ensures that the reward belongs to the set
$\Sc$ if the set $\Sc$ is response-satisfiable. The existence of such
a relation was crucial in obtaining the concept of the Von Neumann
winner described in Section~\ref{sec:prob} for the uni-criterion
problem.

However, such a statement fails to hold in the general vector-valued
case (see~\cite{abernethy2011} for a counterexample). In order to
overcome this limitation, Blackwell~\cite{blackwell1956} defined the
notion of approachability as follows.
\begin{definition}[Blackwell's Approachability]
  Given a Blackwell game $(\X, \Y, \rew, \Sc)$, we say that a set
  $\Sc$ is approachable if there exists an algorithm $\alg$ which
  selects points in $\X$ such that for any sequence $\y_1, \ldots \y_t
  \in \Y$,
  \begin{equation*}
    \lim_{T \to \infty }\distf \left(\frac{1}{T}\sum_{t=1}^T\rew(\x_t,
    \y_t), \Sc \right) \rightarrow 0\;,
  \end{equation*}
  where $\x_t = \alg(\y_1, \ldots, \y_{t-1})$ is the algorithm's play
  at time $t$ for some distance function $\distf$.
\end{definition}
Blackwell's celebrated theorem guarantees that any set $\Sc$ is
approachable if and only if it is response-satisfiable. This means
that while no single choice of action in the set $\X$ can guarantee a
response in the set $\Sc$, there is an algorithm that ensures that in
the repeated game, the average reward vector approaches the set $\Sc$
for any choice of opponent play.

Note that our definition of achievability is a stronger requirement
than Blackwell's approachability. While approachability requires the
time-averaged payoff in a repeated game to belong to the pre-specified
set $\Sc$, achievability requires the same to be true in a single-shot
play of the game. Indeed, as the following lemma shows, one can
construct examples of multi-criteria preference problems which are
approachable but not achievable.

\begin{proposition}[Approachability does not imply achievability]
  \label{lem:app_ach}
There exists a preference tensor $\pref \in \prefS_{\adim, \cdim}$ and
a target set $\Sc \subset [0,1]^\cdim$ such that \\
\begin{enumerate}[label=(\alph*)]
  \item For the
Blackwell game given by $(\simplex_\adim, \simplex_\adim, \pref,
\Sc)$, the set $\Sc$ is approachable, and
\item The set $\Sc$ is not
  achievable with respect to $\pref$.
\end{enumerate}
\end{proposition}
\begin{proof}
  We will consider an example in a 2-dimensional action space with 2
  criteria. Consider the preference matrix given by:
  \begin{equation}\label{eq:two_ex}
    \pref^{1} =
    \begin{bmatrix}
  \frac{1}{2} &
  \vphantom{\frac{1}{2}}1\vspace{1mm}\\ \vphantom{\frac{1}{2}}0&\frac{1}{2}\\
  \end{bmatrix}
  \qquad \text{and} \qquad
  \pref^{2} =
  \begin{bmatrix}
  \frac{1}{2}&\vphantom{\frac{1}{2}}0\vspace{1mm}\\
  \vphantom{\frac{1}{2}}1&\frac{1}{2}\\
  \end{bmatrix}\;,
  \end{equation}
  along with the convex set $\Sc = [\frac{1}{2}, 1]^2$. The tensor
  $\pref$ represents the strongest possible trade-off between the two
  objects: Object $1$ is preferred over $2$ along the first criterion
  while the reverse is true for the second criterion.

  The Blackwell game given by $(\simplex_\adim, \simplex_\adim, \pref,
  \Sc)$ can indeed be shown to be approachable. The set $\Sc$ is
  response-satisfiable since for every strategy $\y \in
  \simplex_{\adim}$ chosen by the column player, the choice of $\x =
  \y$ would yield a reward vector $\pref(\x, \y) = [\frac{1}{2},
    \frac{1}{2}] \in \Sc$. Then, by Blackwell's
  theorem~\cite{blackwell1956}, the set $\Sc$ is approachable.

  In contrast, consider any choice of distribution $\robj_1 = [p,
    1-p]$ for the multi-criteria preference problem. The corresponding
  score vectors for responses $\ic = 1, 2$ are given by:
  \begin{equation*}
    \rvec_1 = \pref(\robj_1, \ic = 1) = \left[\frac{p}{2}, 1-
      \frac{p}{2} \right] \qquad \text{and} \qquad \rvec_2 =
    \pref(\robj_1, \ic = 2) = \left[ \frac{1}{2} + \frac{p}{2},
      \frac{1}{2} - \frac{p}{2} \right]\;.
  \end{equation*}
  For any choice of the parameter $p \in [0,1]$, one cannot have both
  $\rvec_1$ and $\rvec_2$ simultaneously belong to the set
  $\Sc$. Hence, we have that the set $\Sc$ is not achievable with
  respect to $\pref$.

  This example can be extended to any arbitrary dimension $\cdim$ by
  extending the tensor to have $\pref^j$ equal to the all-half matrix
  for any criterion $j > 2$ and the target set to be $\Sc=
  [\frac{1}{2}, 1]^\cdim$. Similarly, in order to extend the example
  to any dimension, consider the preference tensor (for $\cdim = 2$)
  \begin{equation*}
    \pref^1_{\adim} = \begin{bmatrix} \pref^1 & \prefhalf & \cdots &
      \prefhalf \vspace{0.1in}\\ \prefhalf & \pref^1 & \cdots &
      \prefhalf\vspace{0.1in}\\ \vdots & \cdots & \ddots &
      \vdots \vspace{0.1in} \\ \prefhalf & \prefhalf & \cdots &
      \pref^1\vspace{0.1in}\\
  \end{bmatrix} \quad \text{and} \quad   \pref^2_{\adim} = \begin{bmatrix}
    \pref^2 & \prefhalf & \cdots &
    \prefhalf \vspace{0.1in}\\ \prefhalf & \pref^2 & \cdots &
    \prefhalf\vspace{0.1in}\\ \vdots & \cdots & \ddots &
    \vdots \vspace{0.1in} \\ \prefhalf & \prefhalf & \cdots &
    \pref^2\vspace{0.1in}\\
  \end{bmatrix},
  \end{equation*}
  with the smaller matrices $\pref^1$ and $\pref^2$ from
  equation~\eqref{eq:two_ex} at the diagonal and $\prefhalf$ denoting
  the all-half tensor of the appropriate dimension. A similar
  calculation as for the $\adim = 2$ case yields that the set $\Sc$ is
  not achievable. This establishes the required claim.
\end{proof}

\section{Proof of main results}
\label{app:proof_main}
In this section, we provide proofs of all the results stated in the
main paper. Appendix~\ref{app:add_res} to follow collects some
additional results and their proofs.

\subsection{Proof of Proposition~\ref{prop:lin_uni}}
We establish both parts of the proposition separately.

\subsubsection{Proof of part (a)}

For any  weight vector $\wght \in \simplex_{\cdim}$, consider the set
\begin{equation*}
\Sc_\wght = \left\lbrace \rvec \in [0,1]^\cdim\; | \;
\inner{\wght}{\rvec} \geq 1/2 \right\rbrace.
\end{equation*}
The set $\Sc_\wght$ is clearly convex. Indeed, for any two vectors
$\rvec_1, \rvec_2 \in \Sc_\wght$ and any scalar $\alpha \in [0,1]$, we
have
\begin{equation*}
  \inner{\wght}{\alpha \rvec_1 + (1-\alpha)\rvec_2} =
  \alpha\inner{\wght}{\rvec_1} + (1-\alpha) \inner{\wght}{\rvec_2} \in
  \left[\frac{1}{2}, 1 \right].
\end{equation*}
It is straightforward to verify that the set $\Sc_\wght$ is also
monotonic with respect to the orthant ordering.

We now show that a von Neumann winner $\pi^*$ of the
(single-criterion) preference matrix \mbox{$\pref_\wght \defn
  \pref(w)$} can be written as $\pi(\pref, \Sc_\wght, \| \cdot \|)$
for an arbitrary choice of norm $\| \cdot \|$.  For each
$\widetilde{\robj} \in \simplex_{\adim}$, we have
\begin{equation*}
  \inner{\wght}{\pref(\robj^*, \widetilde{\robj})} = \sum_{\jj \in
    [\cdim]} \wght_\jj \pref^\jj(\robj^*, \widetilde{\robj}) =
  \pref_\wght(\robj^*, \widetilde{\robj})) \stackrel{\1}{\geq}
  \frac{1}{2}\;,
\end{equation*}
where the inequality $\1$ follows since $\robj^*$ is a von Neumann
winner for the matrix $\pref_\wght$. Thus, we have the inclusion
$\pref(\robj^*, \widetilde{\robj}) \in \Sc_\wght$ for all
$\widetilde{\robj} \in \simplex_ {\adim}$, so that
$\max_{\widetilde{\robj} \in \simplex_ {\adim}} \distf(\pref(\robj^*,
\widetilde{\robj}), \Sc_{\wght}) = 0$ for any distance metric $\rho$.
Consequently, we have
\begin{align*}
\pi^* \in \arg \min_{\robj \in \simplex_\cdim}
\pmax_{\widetilde{\robj} \in \simplex_ {\adim}} \distf(\pref(\robj,
\widetilde{\robj}), \Sc_{\wght}),
\end{align*}
which establishes the claim for part (a).  \qed

\subsubsection{Proof of part (b)}

Consider the multi-criteria preference instance given by target set
$\Sc = [\frac{1}{2}, 1]^{\cdim}$, the $\ell_\infty$ distance function
and the preference tensor $\pref$
\begin{equation*}
  \pref^1 = \begin{bmatrix} \frac{1}{2} & 1\vspace{1mm}\\ 0 &
    \frac{1}{2}
  \end{bmatrix}, \quad  \pref^2 = \begin{bmatrix}
    \frac{1}{2} & 0\vspace{1mm}\\ 1 & \frac{1}{2}
    \end{bmatrix},\quad \text{and} \quad \pref^\jj = \begin{bmatrix}
    \frac{1}{2} & \frac{1}{2}\vspace{1mm}\\
    \frac{1}{2} & \frac{1}{2}
    \end{bmatrix}
\end{equation*}
The unique Blackwell winner for this instance $(\pref, \Sc,
\norm_\infty)$ is given by
\begin{equation}
\label{eq:bw_wght}
\underbrace{\robj(\pref, \Sc, \norm_\infty)}_{\robjs} = \left[1/2,
  {1}/{2}\right].
\end{equation}
For any weight $\wght \in [0,1]^\cdim$, consider the von Neumann winners corresponding to the weighted matrices $\pref_{\wght}$
\begin{equation}\label{eq:vnw_wght}
    \robj(\pref_{\wght}, [1/2, 1], |\cdot|) = \begin{cases} [1, 0]
      \quad &\text{for } \wght \text{ s.t. } \pref_{\wght}(1, 2) >
      0.5\\ [0, 1] \quad &\text{for } \wght \text{ s.t. }
      \pref_{\wght}(1, 2) < 0.5\\ \pi \in \simplex_2 \quad
      &\text{otherwise }
    \end{cases}.
\end{equation}
Comparing equations~\eqref{eq:bw_wght} and~\eqref{eq:vnw_wght}
establishes the required claim.  \qed

\subsection{Proof of Theorem~\ref{thm:plugin_upper}}

Let us use the shorthand $\widetilde{\pi} \defn \pi(\preftil)$. We
begin by decomposing the desired error term as
\begin{align*}
&\DeltaP (\widetilde{\pi}, \robjs) \\ &=
  \underbrace{\vg(\widetilde{\pi}; S, \pref, \| \cdot \| ) -
    \vg(\widetilde{\pi}; S, \preftil, \| \cdot \|
    )}_{\text{Perturbation error at } \widetilde{\robj}} +
  \underbrace{\vg(\widetilde{\pi}; S, \preftil, \| \cdot \| ) - \vg
    (\pi^*; S, \preftil, \| \cdot \| )}_{\leq 0} + \underbrace{\vg
    (\pi^*; S, \preftil, \| \cdot \| ) - \vg(\pi^*; S, \pref, \| \cdot
    \| )}_{\text{Perturbation error at } \robj^*}
\end{align*}
In order to obtain a bound on the perturbation errors, note that for
any distribution $\pi$, we have
\begin{align}\label{eq:bnd_est_err_1}
  \vg(\pi; \Sc, \pref, \| \cdot \| ) - \vg(\pi; \Sc, \preftil, \|
  \cdot \| ) &= \max_{\ir}[\distf(\pref(\robj, \ir), \Sc)] -
  \max_{\ic}[\distf(\preftil(\robj, \ic),
    \Sc)]\nonumber\\ &\stackrel{\1}{\leq}
  \max_{\ii}[\distf(\pref(\robj, \ii), \Sc) - \distf (\preftil (\robj,
    \ii), \Sc)],
\end{align}
where step $\1$ follows by setting the $\ic$ equal to $\ir$. Noting
that the distance is given by the $\ell_q$ norm, we have
\begin{align*}
     \vg(\pi; \Sc, \pref, \| \cdot \| ) - \vg(\pi; \Sc, \preftil, \|
     \cdot \| ) &\leq \max_{\ii}[\min_{\z_1 \in \Sc}\|\pref(\robj,
       \ii) - \z_1\|_q - \min_{\z_2 \in \Sc}\|\preftil(\robj, \ii) -
       \z_2\|_q]\\ &\stackrel{\1}{\leq} \max_{\ii}[\|\pref(\robj, \ii)
       - \preftil(\robj, \ii) \|_q],
\end{align*}
where the inequality $\1$ follows by setting $\z_2$ equal to
$\z_1$. Taking a supremum over all distributions $\pi$ completes the
proof.  \qed

\subsection{Proof of Corollary~\ref{cor:linfty}}
By Theorem~\ref{thm:plugin_upper}, it suffices to provide a bound on
the quantity $\max_\ii\|\pref(\cdot, \ii) - \hpref(\cdot, \ii))
\|_{\infty, \infty}$ for the plug-in preference tensor $\hpref$. Now
by definition, we have
\begin{equation*}
  \max_\ii\|\pref(\cdot, \ii) - \hpref(\cdot, \ii)) \|_{\infty,
    \infty} = \max_{\ir, \ic, \jj}|\pref^\jj(\ir, \ic) -
  \hpref^\jj(\ir, \ic)|\;.
\end{equation*}
For each $i = (i_1, i_2, j)$ representing some index of the tensor,
let $ N_i \defn \# \{\ell \mid \eta_\ell = i \} $ denote the number of
samples observed at that index. Since $N_i$ can be written as a sum of
i.i.d. Bernoulli random variables, applying the Hoeffding bound yields
\begin{align*}
\Pr \left\{ \left| N_i - \frac{n}{\adim^2 \cdim} \right| \geq
c\sqrt{\frac{n \log (c / \delta)}{\adim^2 \cdim}}\right\} \leq \delta
\text{ for each } \delta \in (0, 1).
\end{align*}
Note that we also have $n \geq c_0 \adim^2 \cdim \log (c_1 \adim /
\delta)$ by assumption. For a large enough choice of the constants
$(c_0, c_1)$, applying the union bound yields the sequence of sandwich
relations
\begin{align} \label{eq:Ni}
\frac{n}{2 \adim^2 \cdim} \leq N_i \leq \frac{3n}{2\adim^2 \cdim}
\quad \text{ for all indices } i \text{ with probability greater than
} 1- \delta.
\end{align}
Furthermore, conditioned on $N_i$ (for $i = (i_1, i_2, j)$), the Hoeffding bound yields the relation
\begin{align*}
\Pr \left\{ |\pref^\jj(\ir, \ic) - \hpref^\jj(\ir, \ic)| \geq
c\sqrt{\frac{\log (c / \delta)}{N_i}} \right\} \leq \delta \text{ for
  each } \delta \in (0, 1).
\end{align*}
Putting this together with a union bound, we have
\begin{align} \label{eq:Hoeff-cond}
\Pr \left\{ \max_{i_1, i_2, j} |\pref^\jj(\ir, \ic) - \hpref^\jj(\ir, \ic)| \geq c\sqrt{\frac{\log (c \adim^2 \cdim / \delta)}{\min_i N_i}} \right\} \leq \delta.
\end{align}
Combining inequalities~\eqref{eq:Ni} and~\eqref{eq:Hoeff-cond} with a
final union bound completes the proof.  \qed

\subsection{Proof of Theorem~\ref{thm:lower_passive}}
Suppose throughout that $\cdim \geq 2$, and recall the axis-aligned
convex target set $\Sc_0 = [\frac{1}{2}, 1]^\cdim$. We split our proof
into two cases depending on whether $\adim$ is even or odd.

\paragraph{Case 1: $\adim$ even.}
We use Le Cam's method and construct two problem instances with
preference tensors given by $\pref_0$ and $\pref_1$. Two key elements
in the construction are the following $2\times 2\times 2$ tensors,
which we denote by $\prefbl$ and $\wtprefbl$, respectively. Their
entries are given by
  \begin{equation*}
    \prefbl^1 = \begin{bmatrix}
      \frac{1}{2} & \frac{1}{2}+\gap \\
      \frac{1}{2} - \gap & \frac{1}{2}
      \end{bmatrix}
      \quad , \quad
      \prefbl^2 = \begin{bmatrix}
        \frac{1}{2} & \frac{1}{2}-\gap \\
        \frac{1}{2} + \gap & \frac{1}{2}
        \end{bmatrix},
  \end{equation*}
\begin{equation*}
    \wtprefbl^1 = \begin{bmatrix} \frac{1}{2} &
      \frac{1}{2}+\frac{\gap}{\adim} \\ \frac{1}{2} -
      \frac{\gap}{\adim} & \frac{1}{2}
      \end{bmatrix}
      \quad \text{and} \quad
      \wtprefbl^2  = \begin{bmatrix}
        \frac{1}{2} & \frac{1}{2}-\frac{\gap}{\adim} \\
        \frac{1}{2} + \frac{\gap}{\adim} & \frac{1}{2}
        \end{bmatrix}\;.
  \end{equation*}

Note that these tensors are parameterized by a scalar $\gap \in [0,
  1/2]$, whose exact value we specify shortly.  Also denote by
$\pref_{1/2}$ the $2 \times 2 \times 2$ all-half tensor. We are now
ready to construct the pair of $\adim \times \adim \times \cdim$
preference tensors $(\pref_0, \pref_1)$.

In order to construct tensor $\pref_0$, we specify its entries on the
first two criteria according to
  \begin{equation}\label{eq:pref1}
    \pref_0^{1:2} = \begin{bmatrix}
    \prefhalf & \prefhalf & \cdots & \prefhalf \vspace{0.1in}\\
    \prefhalf & \prefbl & \cdots & \prefhalf\vspace{0.1in}\\
    \vdots & \cdots & \ddots  & \vdots \vspace{0.1in} \\
    \prefhalf & \prefhalf & \cdots & \prefbl\vspace{0.1in}\\
  \end{bmatrix},
  \end{equation}
  and set the entries on the remaining $k - 2$ criteria to $1/2$.

On the other hand, the first two criteria of the tensor $\pref_1$ are given by
  \begin{equation}\label{eq:pref2}
    \pref_1^{1:2} = \begin{bmatrix} \wtprefbl & \prefhalf & \cdots &
      \prefhalf\vspace{0.1in}\\ \prefhalf & \prefbl & \cdots &
      \prefhalf\vspace{0.1in}\\ \vdots & \cdots & \ddots &
      \vdots \vspace{0.1in} \\ \prefhalf & \prefhalf & \cdots &
      \prefbl\vspace{0.1in}\\
  \end{bmatrix}\;,
  \end{equation}
  with the entries on the remaining $k - 2$ criteria once again set
  identically to $1/2$.

  Note that the tensors $\pref_0$ and $\pref_1$ only differ on the
  first $2\times 2\times 2 $ block. Furthermore, the following lemma
  provides an exact calculation of the values $\min_{\pi} v(\pi;
  \pref_0, \Sc_0, \| \cdot \|_\infty)$ and \mbox{$\min_{\pi} v(\pi;
    \pref_1, \Sc_0, \| \cdot \|_\infty)$}.
\begin{lemma} \label{lem:vals}
We have
\begin{align*}
\val_0 \defn \min_{\pi} v(\pi; \pref_0, \Sc_0, \| \cdot \|_\infty) = 0
\quad \text{ and } \quad \val_1 \defn \min_{\pi} v(\pi; \pref_0,
\Sc_0, \| \cdot \|_\infty) = \frac{\gap}{3\adim - 2}.
\end{align*}
\end{lemma}

Given samples from these two instances, we now use Le Cam's
lemma~\cite[see][Chap 2]{tsybakov2008} to lower bound the minimax risk
as
  \begin{equation}\label{eq:low}
    \minmax_{n, \adim, \cdim} (\Sc_0, \| \cdot \|_{\infty} ) \geq
    \frac{|\val_0 - \val_1|}{2}\left( 1- \|\prob_0^\samp -
    \prob_1^\samp \|_{\text{TV}}\right) = \frac{\gap}{2(3\adim - 2)}
    \left( 1- \|\prob_0^\samp - \prob_1^\samp \|_{\text{TV}}\right),
  \end{equation}
where $\prob_0^\samp$ and $\prob_1^\samp$ are the probability
distributions induced on sample space by the passive sampling strategy
applied to the tensor $\pref_0$ and $\pref_1$, respectively.

Using Pinsker's inequality, the decoupling property for KL divergence
and the fact that that $\text{KL}(P\|Q) \leq \chi^2(P\|Q)$, we have
  \begin{equation}\label{eq:tvchi}
    \|\prob_0^\samp - \prob_1^\samp \|_{\text{TV}} \leq
    \sqrt{\frac{\samp}{2}\text{KL}(\prob_1\|\prob_0)} \leq
    \sqrt{\frac{\samp}{2}\chi^2(\prob_1\|\prob_0)}\;.
  \end{equation}
The chi-squared distance between the two distributions $\prob_0$ and
$\prob_1$ is given by
  \begin{equation*}
    \chi^2(\prob_1 \|\prob_0) = \frac{1}{\adim^2\cdim}\sum_{(\ir, \ic,
      \jj)} \left( \frac{\pref_1^\jj(\ir, \ic)}{\pref_2^\jj(\ir, \ic)}
    -1\right)^2 \stackrel{\1}{=} \frac{2}{\adim^2\cdim}\; \left(
    \left(\frac{2\gamma}{\adim} \right)^2 +
    \left(-\frac{2\gamma}{\adim}\right)^2\right)=
    \frac{16\gamma^2}{\adim^4\cdim}\;,
  \end{equation*}
where step $\1$ follows from the fact that $\pref_1$ and $\pref_2$
differ only in 4 entries and that the passive sampling strategy
samples each index uniformly at random. Putting together the pieces,
we have:
\begin{align*}
    \minmax_{n, \adim, \cdim} (\Sc_0, \| \cdot \|_\infty ) &\geq
    \frac{\gap}{2(3\adim-2)}\left(1-
    \sqrt{\frac{\samp}{2}\;\frac{16\gap^2}{\adim^4\cdim}}\right)\stackrel{\2}{=}
    \frac{1}{48\sqrt{2}}\sqrt{\frac{\adim^2\cdim}{\samp}}.
  \end{align*}
where step $\2$ follows by setting $\gap^2 =
\frac{\adim^4\cdim}{32\samp}$ and using the fact that $3\adim-2 \leq
3\adim$. Note that since we require $\gap^2\leq \frac{1}{4}$, the
above bound is valid only for $\samp \gtrsim \adim^4 \cdim$. This
concludes the proof for even $\adim$. \\

\paragraph{Case 2: $\adim$ odd.} By assumption, we have $d \geq 5$.
In this case, we construct $\pref_0$ and $\pref_1$ exactly as before,
but replace $\prefbl$ in the last two rows of both $\pref_0$ and
$\pref_1$ with the following modified $3\times 3\times 2$ tensor:
  \begin{equation*}
    \pref^1_{\sf cr, 3} = \begin{bmatrix} \frac{1}{2} &
      \frac{1}{2}+\gap & \frac{1}{2} - \gap\\ \frac{1}{2} - \gap &
      \frac{1}{2} & \frac{1}{2} - \gap\\ \frac{1}{2} + \gap &
      \frac{1}{2} + \gap & \frac{1}{2}
      \end{bmatrix}
      \quad \text{and} \quad
      \pref^2_{\sf cr, 3} = \begin{bmatrix}
        \frac{1}{2} & \frac{1}{2}-\gap & \frac{1}{2} + \gap\\
        \frac{1}{2} + \gap & \frac{1}{2} & \frac{1}{2} + \gap\\
        \frac{1}{2} - \gap  & \frac{1}{2} - \gap & \frac{1}{2}
        \end{bmatrix}.
  \end{equation*}
By mimicking its proof, it can be verified that this modification
ensures that the corresponding values $\val_0$ and $\val_1$ still
satisfy Lemma~\ref{lem:vals}. Thus, the lower bound remains unchanged
up to constant factors.  \qed

\subsubsection{Proof of Lemma~\ref{lem:vals}}

Let us compute the two values separately.

\paragraph{Computing $\val_0$.}
  The choice of distribution $\robjs = [1, 0, \ldots, 0]$ yields the
  score vector $[1/2, 1/2, \ldots, 1/2]$, which is in the set
  $\Sc_0$. Thus, we have $\val_0 = 0$.

\paragraph{Computing $\val_1$.}
Note that the optimal distribution $\robj^*$ achieving the value
$\val_1$ will be of the form
\begin{equation*}
\robjs = [p/2, p/2, (1-p)/(\adim-2), \ldots, (1-p)/(\adim-2)] \;\; \text{for some } p \in [0,1].
\end{equation*}
This follows from the symmetry in the preference tensor for row
objects ranging from $3$ to $\adim$. Given such a distribution
$\robjs$, the distance of the reward vector from the set $\Sc_0$ is
given by
\begin{equation*}
  \inf_{\z \in \Sc}\|\pref(\robj^*, \ic) - \z\|_\infty = \begin{cases}
    \frac{\gap p}{2\adim} \quad \ic = 1,2\vspace{0.1in}\\ \frac{\gap
      (1-p)}{\adim-2} \quad \text{o.w.}
\end{cases}.
\end{equation*}
Thus, for any value of $p > 2\adim/(3\adim-2)$, the distance is
maximized for $\ic \in \{1, 2\}$, and yields a value $\gap p
/(2\adim)$. On the other hand, for $p< 2\adim/(3\adim-2)$, the
maximizing index is $i_2 \geq 3$, and the maximizing value is
$\gap(1-p)/(\adim-2)$. Optimizing these values for $p$ yields the
claim.  \qed

\subsection{Instance-specific lower bounds}

In this section, we give a formal statement of
Proposition~\ref{prop:lower_gens} along with its proof.  We begin by
defining some notation. For any $\alplb, \betlb \in [-\frac{1}{2},
  \frac{1}{2}]$ and choice of criteria $\jj_1, \jj_2 \in [\cdim]$, we
define the tensor $\prefab^{(\jj_1, \jj_2)} \in [0,1]^{2\times 2
  \times \cdim}$ as
\begin{equation*}
  \prefab^{\jj_1} = \begin{bmatrix} \frac{1}{2} &
    \frac{1}{2}+\alplb\vspace{1mm} \\ \frac{1}{2} - \alplb &
    \frac{1}{2}
    \end{bmatrix}
    \;, \quad \prefab^{\jj_2} = \begin{bmatrix} \frac{1}{2} &
      \frac{1}{2}+\betlb\vspace{1mm}\\ \frac{1}{2} - \betlb &
      \frac{1}{2}
      \end{bmatrix}
      \quad \text{and} \quad \prefab^\jj = \begin{bmatrix} \frac{1}{2}
        & \frac{1}{2}\vspace{1mm}\\ \frac{1}{2} & \frac{1}{2}
        \end{bmatrix}\; \text{for } \jj \neq \{\jj_1, \jj_2 \}\;.
\end{equation*}
Further, we denote by $\prefhalf$ the all-half tensor whose dimensions
may vary depending on the context. Any distribution $\robj$ over the
two objects can be parameterized by a value $\ld \in [0,1]$ with $\ld$
being the probability placed on the first object and $1-\ld$ the
probability on the second object. We will consider the distance
function given by the $\ell_\infty$ norm. Given this distance
function, we overload our notation for the value
\begin{equation}\label{eq:val_lbS}
  \vg(\ld;\prefab^{(\jj_1, \jj_2)}, \Sc) = \max_{\ii}
     [\distf(\prefab^{(\jj_1, \jj_2)}(\ld, \ii), \Sc)] \quad
     \text{and} \quad \val(\prefab^{(\jj_1, \jj_2)};\Sc) = \min_{\ld}
     \vg(\ld;\prefab^{(\jj_1, \jj_2)};\Sc)\;.
\end{equation}
We now state our main assumption for the score set $\Sc$ which allows
us to formulate our lower bound.
\begin{assumption}
\label{ass:exist_lb}
There exists a pair of criteria $(\jj_1, \jj_2)$, values $\alpz\in (0,
\frac{1}{2}]$ and $\betz \in [-\frac{1}{2}, 0]$, and a gap parameter
  $\gap > 0$ such that $$\val(\prefhalf;\Sc) + \gap \leq
  \val(\prefz^{(\jj_1, \jj_2)};\Sc)$$ for the all-half tensor
  $\prefhalf \in [0,1]^{2\times 2\times \cdim}$.
\end{assumption}
The assumption above indicates that there exists a pair of criteria
along which one can observe some sort of trade-off when they interact
with the underlying score set $\Sc$. The preference tensor
$\prefz^{(\jj_1, \jj_2)}$ captures this trade-off and the gap
parameter $\gap$ quantifies it. Going forward, we assume without loss
of generality that $(\jj_1, \jj_2) = (1, 2)$ and drop the dependence
of the tensor on these indices, writing $\prefz \equiv \prefz^{(1,
  2)}$. The following lemma indicates the importance of the special
values of $(\alplb, \betlb) = (0, 0)$ for which $\pref_{0, 0} =
\prefhalf$.

\begin{lemma}\label{lem:half_best}
For any $\alplb, \betlb \in [-\frac{1}{2}, \frac{1}{2}]$, we have
$\val(\pref_{0, 0};\Sc) \leq \val(\prefab;\Sc)$.
\end{lemma}

The above lemma establishes that for any set, the value attained by
setting $(\alpz, \betlb) = (0,0)$ will be lower than any other setting
of the same parameters. For any parameter $\mixd \in [0, 1]$, denote
by $\prefm{\mixd}$ the weighted tensor
\begin{equation*}
\prefm{\mixd} \defn (1-\mixd)\pref_{0, 0} + \mixd \prefz.
\end{equation*}

In order to understand the value $\val(\prefm{\mixd};\Sc)$, we
establish the following structural lemma which gives us insight into
how this value varies as a function of the parameter $\mixd \in
[0,1]$.
\begin{lemma}\label{lem:val_pwl}
  Consider a target set $\Sc$ that is given by an intersection of
  $\numh$ half-spaces. Then, the value function
  $\val(\prefm{\mixd};\Sc)$ is a piece-wise linear and continuous
  function of $\mixd \in [0,1]$ with at most $4\numh$ pieces.
\end{lemma}
The above lemma states that the value $\val(\prefm{\mixd};\Sc)$ is a
piece-wise linear function of $\mixd$. Consider the first such piece
which has a non-zero slope. Such a line has to exist since
$\val(\prefm{\mixd})$ is continuous in $\mixd$ and we have
$\val(\prefm{0}) < \val(\prefm{1})$. Also, this slope has to be
positive since we know from Lemma~\ref{lem:half_best} that
$\val(\prefm{0}) \leq \val(\prefm{\mixd})$ for any $\mixd \in [0,
  1]$. Denote the starting point of this line by $\mixd_0$ and the
corresponding slope by $\slope_0$, and observe that the value
$\val(\prefm{\mixd_0}) = \val(\prefm{0})$. With this notation, we now
proceed to prove the lower bound on sample complexity for any
polyhedral target score set $\Sc$.
\begin{proposition}[Formal]\label{prop:lower_gens_formal}
  Suppose that we have a valid polyhedral target set $\Sc$ satisfying
  Assumption~\ref{ass:exist_lb} with parameters $(\alpz, \betz)$.
  Then, there exists a universal constant $c$ such that for all $\adim
  \geq 4$, $\cdim \geq 2$, and $\samp \geq
  \frac{\adim^2\cdim}{\mixdb^2}\frac{(1/2 - \mixd_0\alpz)^2}{\alpz^2 +
    \betz^2}$, we have
  \begin{equation} \label{eq:minimax-lb-formal}
    \minmax_{\samp, \adim, \cdim} (\Sc, \norm_\infty ) \geq \const
    \frac{\slope_0 (\frac{1}{2} - \mixd_0\alpz)}{\sqrt{\alpz^2 +
        \betz^2}}\sqrt{\frac{\adim^2\cdim}{\samp}}\;.
  \end{equation}
\end{proposition}
\begin{proof}
For this proof, we focus on the case when the number of criteria
$\cdim$ is even. The proof for the case when $\cdim$ is odd can be
obtained similar to the proof of Theorem~\ref{thm:lower_passive}.

We use Le Cam's method for obtaining a lower bound on the minimax
value and construct the lower bound instances using the tensor given
by $\prefm{\mixd}$.  For some $\mixd \in [0,1]$ (to be fixed later),
consider the parameter $\mixd_1 = \mixd_0 + \mixd$. Using these values
of $\mixd_0$ and $\mixd_1$, we create the following two instances
$\pref_0$ and $\pref_1$:
  \begin{equation*}
    \pref_0 = \begin{bmatrix}
     \prefm{\mixd_0} & \prefhalf & \cdots & \prefhalf \vspace{0.1in}\\
    \prefhalf & \prefz & \cdots & \prefhalf\vspace{0.1in}\\
    \vdots & \cdots & \ddots  & \vdots \vspace{0.1in} \\
    \prefhalf & \prefhalf & \cdots & \prefz\\
  \end{bmatrix}\quad \text{and} \quad
  \pref_1 = \begin{bmatrix} \prefm{\mixd_1} & \prefhalf & \cdots &
    \prefhalf \vspace{0.1in}\\ \prefhalf & \prefz & \cdots &
    \prefhalf\vspace{0.1in}\\ \vdots & \cdots & \ddots &
    \vdots \vspace{0.1in} \\ \prefhalf & \prefhalf & \cdots & \prefz\\
\end{bmatrix}\;,
\end{equation*}
where $\prefz$ is as given by Assumption~\ref{ass:exist_lb}. The
following lemma now shows that that there exists a small enough
$\mixdb$ such that the value function $\val(\prefm{\mixd};\Sc)$ is
linear in the range $\mixd \in [\mixd_0, \mixd_1]$.

\begin{lemma}\label{lem:mixd_exist}
  There exists a $\mixdb \in (0,1)$ such that for all $\mixd \in [0,
    \mixdb]$ and $\mixd_1 = \mixd_0 + \mixd$, we have
  \begin{enumerate}[label=(\alph*)]
    \item The value $\val(\prefm{\mixd_1};\Sc) = \val(\prefm{\mixd_0};\Sc) + \mixd\slope_0$.
    \item The minimizer $\robjs_1$ for $\pref_1^*$ is given by
      $\robjs_1 = [\ld_0, 1-\ld_0, 0 \ldots, 0]$.
  \end{enumerate}
\end{lemma}
We defer the proof of this lemma to the end of the section. Thus, for
a small enough value of $\mixd \in [0, \mixdb]$, we have
$|\val(\pref_0) - \val(\pref_1)| = \mixd \slope_0$. As was shown in
the proof of Theorem~\ref{thm:lower_passive}, the minimax rate is
lower bounded as
\begin{equation}\label{eq:low_genS}
   \minmax_{\samp, \adim, \cdim} (\Sc, \norm_\infty ) \geq \frac{|\val(\pref_0) - \val(\pref_1)|}{2}\left( 1- \|\prob_0^\samp - \prob_1^\samp \|_{\text{TV}}\right) \geq \frac{\mixd\slope_0}{2}\left( 1- \sqrt{\frac{\samp}{2}\chi^2(\prob_1\|\prob_0)}\right) \;,
\end{equation}
where $\prob_0^\samp$ and $\prob_1^\samp$ are the probability distributions induced on sample space by the passive sampling strategy and the preference tensor $\pref_0$ and $\pref_1$ respectively. In order to obtain the requisite lower bound, we proceed to compute an upper bound on the chi-squared distance between the two distributions $\prob_0$ and $\prob_1$ as
\begin{align*}
\chi^2(\prob_1 \|\prob_0) &= \frac{1}{\adim^2\cdim}\sum_{(\ir, \ic, \jj)} \left( \frac{\pref_1^\jj(\ir, \ic)}{\pref_0^\jj(\ir, \ic)} -1\right)^2 \\
&\stackrel{\1}{\leq} \frac{2}{\adim^2\cdim}\; \left( \left(\frac{\alpz^2\mixd^2}{(\frac{1}{2} - \mixd_0\alpz)^2} \right) + \left(\frac{\betz^2\mixd^2}{(\frac{1}{2} + \mixd_0\betz)^2} \right) \right) \\
&\stackrel{\2}{\leq} \frac{2\mixd^2}{\adim^2\cdim}\;\left(\frac{\alpz^2 + \betz^2}{(\frac{1}{2} - \mixd_0\alpz)^2} \right)\;,
\end{align*}
where $\1$ follows from the fact that the instances $\pref_0$ and $\pref_1$ differ only in 4 entries and $\2$ follows from the assumption that $|\alpz| \geq |\betz|$. Now, substituting the value of $\mixd^2 = \frac{\adim^2\cdim}{4\samp}\cdot \frac{(\frac{1}{2} - \mixd_0\alpz)^2}{\alpz^2+\betz^2}$ and using the above bound with equation~\eqref{eq:low_genS}, we have
\begin{equation*}
  \minmax_{\samp, \adim, \cdim} (\Sc, \norm_\infty ) \geq \frac{\slope_0 (\frac{1}{2} - \mixd_0\alpz)}{8\sqrt{\alpz^2 + \betz^2}}\sqrt{\frac{\adim^2\cdim}{\samp}}\;,
\end{equation*}
which holds whenever we have  $\mixd \in [0, \mixdb]$ or equivalently $\samp \geq \frac{\adim^2\cdim}{4\mixdb^2}\frac{(\frac{1}{2} - \mixd_0\alpz)^2}{\alpz^2 + \betz^2}$. This establishes the desired claim.
\end{proof}

\subsubsection{Proof of Lemma~\ref{lem:half_best}}
For any $\alplb, \betlb \in [-\frac{1}{2}, \frac{1}{2}]$, consider the value
\begin{align*}
  \val(\prefab;\Sc) &= \min_{\ld \in [0,1]} \max_{\ii} [\distf(\prefab(\ld, \ii), \Sc)]\\
  &= \min_{\ld \in [0,1]} \max_{\tau \in [0, 1]} [\distf(\prefab(\ld, \tau), \Sc)]\\
  &\stackrel{\1}{\geq} \distf\left(\half{\cdim}, \Sc\right) = \val(\prefhalf;\Sc)\;,
\end{align*}
where $\1$ follows by setting $\tau = \ld$ and $\half{\cdim}$ denotes the vector with each entry set to half. This establishes the claim.
\qed

\subsubsection{Proof of Lemma~\ref{lem:val_pwl}}
Let us denote by $\ld_0$ any minimizer of the value $\vg(\ld; \prefz, \Sc)$ and the two score vectors corresponding to the choices for $\ii$ in equation~\eqref{eq:val_lbS} by $\z_{1, \ii} \defn \prefz(\ld_0, \ii)$. Observe that for $\prefm{\mixd}$, the distribution given by $\ld_0$ is still a minimizer of its value. Further, the score vectors for the two column choices are given by:
\begin{equation*}
  \z_{\mixd, \ii} = (1-\mixd)\half{\cdim} + \mixd \z_{1, \ii} \quad \text{for } \ii = \{1, 2 \}.
\end{equation*}
Recall that the distance function is given by \mbox{$\distf(\z_{\mixd,
    \ii}, \Sc) = \min_{\z \in \Sc} \|\z_{\mixd, \ii} - \z
  \|_\infty$}. Now, the minimizer $\z$ will lie on the closest
hyperplane(s) to the point $\z_{\mixd, \ii}$. In order to establish
the claim, it suffices to show that for any fixed
hyperplane\footnote{We use $\hyper$ to denote the hyperplane and the
half-space induced by it interchangeably.} $\hyper$, the distance
function given by $\distf(\z_{\mixd, \ii}, \hyper)$ is a piece-wise
linear function for $\mixd \in [0,1]$.

Let us consider a point $\z_{\mixd, \ii}$ which does not belong to the half-space given by $\hyper$, since otherwise, the distance to the half-space is $0$. If we have $\distf(\z_{\mixd, \ii}, \hyper) = \dhyp$, then the vector $\z_{\mixd, \ii} + \dhyp\ones_\cdim$ must lie on the hyperplane $\hyper$. This follows from the monotonicity property of the hyperplane $\hyper$.

For any $\mixd = \frac{1}{2}\mixd_1 + \frac{1}{2}\mixd_2$ such that $\z_{\mixd_1, \ii}$ and $\z_{\mixd_2, \ii}$ do not belong to the half-space given by $\hyper$, we have
  \begin{equation*}
    \distf(\z_{\mixd, \ii}) = \frac{1}{2}\underbrace{\distf(\z_{\mixd_1, \ii})}_{\dhyp_1} + \frac{1}{2}\underbrace{\distf(\z_{\mixd_2, \ii})}_{\dhyp_2}\;,
  \end{equation*}
  where the above equality follows since $\z_{\mixd_1, \ii} + \dhyp_1\ones_{\cdim}$ and $\z_{\mixd_2, \ii} + \dhyp_2\ones_{\cdim}$ both lie on the hyperplane $\hyper$ and therefore $\z_{\mixd, \ii} + \frac{\dhyp_1+\dhyp_2}{2}\ones_{\cdim}$ also lies on the hyperplane. Combined with the fact that for any point $\z_{\mixd, \ii}$ which lies in the half-space given by $\hyper$, the distance $\distf(\z_{\mixd, \ii}, \hyper) = 0$, we have that the function $\distf(\z_{\mixd, \ii}, \hyper)$ is a piece-wise linear function with at most 2 linear pieces for $\mixd \in [0,1]$.

  Since $\distf(\z_{\mixd, \ii}, \Sc)$ is a minimum over $\numh$ hyperplanes, this function is itself a piece-wise linear function with at most $2\numh$ pieces. The desired claim now follows from noting that the value function $\val(\prefm{\mixd};\Sc)$is a maximum over two piece-wise linear functions each with at most $2\numh$ pieces.
 \qed

 \subsubsection{Proof of Lemma~\ref{lem:mixd_exist}}
  Consider $\mixd_1 = \mixd_0 + \mixd$ such that $\mixd_0$ and $\mixd_1$ share the same linear piece. This can be guaranteed to hold true for all $\mixd \leq \mixdb_1$ by the piecewise linear nature of the value $\val(\prefm{\mixd})$.

  For part (b) of the claim, let us consider the tensor $\tilde{\pref} = \pref_1(3:, 3:)$ formed by removing the first two rows and columns from the tensor $\pref_1$. Then, from Assumption~\ref{ass:exist_lb}, we have that $\val(\tilde{\pref};\Sc) \geq \val(\prefhalf;\Sc) + \tilde{\gap}$ for some $\tilde{\gap} > 0$. Selecting a value of $\mixdb_2$ such that $\mixdb_2 \slope_0 \leq \tilde{\gap}$, we can ensure that condition (b.) is satisfied.

  Finally, setting $\mixdb = \min(\mixdb_1, \mixdb_2)$ completes the proof.
  \qed

\subsection{Proof of Theorem~\ref{thm:opt}}
Let us prove the two claims of the theorem separately. We use the shorthand $v(\pi) \defn v(\pi; \pref, \Sc, \| \cdot \|)$ for convenience.

\paragraph{Establishing convexity.}
Consider any two distributions $\robj_1, \robj_2 \in \simplex_{\cdim}$ and a scalar $\alpha \in [0,1]$. Since the set $\Sc$ is closed and convex, we have
  \begin{align*}
    \vg(\alpha\robj_1 + (1-\alpha)\robj_2) &=  \pmax_{\ii \in [\adim]}\pmin_{\z \in \Sc} \left[ \distf(\pref(\alpha\robj_1 + (1-\alpha)\robj_2, \ii), \z)\right] \\
    &\stackrel{\1}{=} \pmax_{\ii \in [\adim]}\pmin_{\z_1, \z_2 \in \Sc} \left[ \distf(\alpha\pref(\robj_1, \ii)+ (1-\alpha)\pref(\robj_2, \ii), \alpha\z_1 +(1-\alpha)\z_2)\right] \\
    &\stackrel{\2}{\leq} \pmax_{\ii \in [\adim]}\left(\alpha\cdot\pmin_{\z_1 \in \Sc}\left[\distf(\pref(\robj_1, \ii), \z_1) \right] + (1-\alpha)\cdot\pmin_{\z_2 \in \Sc}\left[\distf(\pref(\robj_2, \ii), \z_2) \right] \right) \\
    &\leq \alpha \vg(\robj_1) + (1-\alpha) \vg(\robj_2)\;,
  \end{align*}
  where $\1$ follows from the convexity of $\Sc$ and linearity of the preference evaluation (Eq.~\eqref{eq:expec_pref}), $\2$ follows from the convexity of the distance function given by $\ell_q$ norm  and $\3$ follows from distributing the max over the two terms. This establishes the first part of the theorem.

\paragraph{Establishing the Lipschitz bound.}
  Consider any two distributions $\robj_1, \robj_2 \in \simplex_\adim$. The difference in their value function can then be upper bounded as
  \begin{align*}
    |\vg(\robj_1) - \vg(\robj_2)| &= | \pmax_{\ir \in [\adim]}[\distf(\pref(\robj_1, \ir), \Sc)] - \pmax_{\ic \in [\adim]}[\distf(\pref(\robj_2, \ic), \Sc)] |\\
    &\stackrel{\1}{\leq}\pmax_{\ii \in [\adim]} \left\vert \distf(\pref(\robj_1, \ii), \Sc) -  \distf(\pref(\robj_2, \ii), \Sc) \right\vert\\
    &= \pmax_{\ii \in [\adim]} | \pmin_{\z_1 \in \Sc}\distf(\pref(\robj_1, \ii), \z_1) -  \pmin_{\z_2 \in \Sc}\distf(\pref(\robj_2, \ii), \z_2) |\\
    &\stackrel{\2}{\leq} \pmax_{\ii \in [\adim]} \pmax_{\z \in \Sc} | \distf(\pref(\robj_1, \ii), \z) -  \distf(\pref(\robj_2, \ii), \z)|\;,
  \end{align*}
  where $\1$ follows from using the inequality $|\max_x f(x) - \max_y g(y)| \leq \max_x|f(x) - g(x)|$ and $\2$ follows through a similar inequality $|\min_xf(x) - \min_yg(y)| \leq \max_x |f(x) - g(x)|$. Since the distance function $\distf$ is specified by the $\ell_q$ norm $\norm_q$, we have
  \begin{align*}
    |\vg(\robj_1) - \vg(\robj_2)| &\leq \pmax_{\ii \in [\adim]}\|\pref(\robj_1, \ii) - \pref(\robj_2, \ii) \|_q \\
    &= \left[\sum_{\jj=1}^\cdim\left( \inner{\robj_1 - \robj_2}{\pref^\jj(\cdot, \ii)}\right)^q\right]^\frac{1}{q}\\
    &\stackrel{\1}{\leq} \cdim^{\frac{1}{q}}\cdot \|\robj_1 - \robj_2 \|_1\;,
  \end{align*}
  where $\1$ follows from an application of H\"older's inequality ($\ell_1-\ell_\infty$) to the inner product \mbox{$\inner{\robj_1 - \robj_2}{\pref^\jj(\cdot,\ii)}$} and the fact that $\pref^\jj(\ir, \ic) \in [0,1]$ for any $(\ir, \ic, \jj)$. This establishes the Lipschitz bound and concludes the proof of the theorem.
\qed

\section{Local asymptotic analysis for plug-in estimator}\label{app:local-asymp}
In this section, we study the adaptivity properties of the plug-in
estimator\footnote{For this section we use the notation $\pihatplug$
and $\pihat$ to interchangeably to denote the plug-in estiamtor.}
$\pihatplug$ and derive upper bounds on the error $\DeltaP(\pihatplug,
\robjs)$ which depend on the properties of the underlying problem
instance $(\pref, \Sc, \distf)$. Contrast this analysis with the upper
bounds obtained in Corollary~\ref{cor:linfty} and the perturbation
result of Dudik et al.~\cite[Lemma 3]{dudik2015} which provides a
worst-case upper bound on the error $\DeltaP$ independent of the
underlying preference tensor $\pref$.

Our focus in this section will be on the uni-criterion setup with $\cdim = 1$ with the target set $\Sc = [\frac{1}{2},1]$ in which case the Blackwell winner coincides with the von Neumann winner.
Recall from Section~\ref{sec:prob-bw} that for the uni-criterion setup, the von Neumann winner for a preference matrix $\pref \in [0,1]^{\adim\times\adim}$ is defined to be the distribution $\robjs$ satisfying
\begin{align}\label{eq:nash}
  \robjs \in \arg \max_{\robj \in \simplex_\adim} \pmin_{i \in [\adim]} \robj^\top \pref e_i\;,
\end{align}
where $e_i$ denotes the basis vector in the $i^{th}$ direction. Observe that the vector $\robjs$ corresponds to the mixed Nash equilibirum (NE) strategy of the zero-sum game with pay-off matrix $\pref$ for the row player (maximizing player). Given this equivalence, we focus on the more general problem of estimating the Nash distribution of a zero-sum game with pay-offs $\pnash \in [0,1]^{\adim \times \adim}$ given sampled access to the matrix $\pnash$.

We consider a slighlty modified passive sampling regime introduced in Section~\ref{sec:main} wherein each sample consists of an observation $y \sim \normal(\pnash_{\ir, \ic}, \sig_{\ir, \ic}^2)$, where the indices $\ir, \ic \sim \mathsf{Unif}([d])$ are sampled independently. We term this the \emph{Gaussian passive sampling model} in contrast to the Bernoulli sampling model considered in the main text. Note that in the asymptotic regime (and with suitable rescaling), the Bernoulli sampling model is equivalent to the Gaussian sampling model with variance \mbox{$\sig_{\ir, \ic}^2 = \pnash_{\ir, \ic}\cdot(1-\pnash_{\ir, \ic})$}.
We further assume that the variances satisfy $\max_{\ir, \ic}\sig_{\ir, \ic}^2 \leq 1$. Given access to $\samp$ samples from this model, we are interested in understanding the performance of the plug-in estimator
\begin{align*}
  \pihatplug \in \arg \max_{\robj \in \simplex_\adim} \pmin_{i \in
    [\adim]} \robj^\top \hpnash_\samp e_i\;,
\end{align*}
where $\hpnash_\samp$ is the empirical estimate of the matrix, defined
analogous to the estimate $\hpref$ in equation~\eqref{eq:emp_pref}. In
particular, we will be interested in obtaining a bound on the error
\begin{align*}
  \DeltaA(\pihatplug, \robjs) \defn \min_{i \in [\adim]} (\robjs)^\top
  \pnash e_i - \min_{i \in [\adim]}\pihatplug^\top \pnash e_i\;,
\end{align*}
which measures the gap in the value obtained when distribution
$\pihatplug$ is played compared with the value obtained by the Nash
distribution $\robjs$. Observe that the optimization problem for
obtaining Nash equilibrium in equation~\eqref{eq:nash} can be written
as the following linear program with decision variables $(\robj, t)$
\begin{equation}\label{eq:nash-lp}\tag{Nash}
  \begin{gathered}
  \max\; t \\ \text{such that } \robj^\top \pnash e_i \geq t \;
  \text{for all } i \in [d],\\ \sum_i \robj_i = 1 \quad \text {and}
  \quad \robj_i \geq 0 \; \text{for all } i \in [d].
  \end{gathered}
\end{equation}
The above linear program has $\adim+1$ variables $(\robj,t)$ and
$2\adim+1$ constraints including one equality constraint. Similarly,
the one can rewrite the objective for the plug-in estimator
$\pihatplug$ as the solution to a perturbed version of the above
linear program
\begin{equation}\label{eq:nash-lp-pert}\tag{Pert}
  \begin{gathered}
  \max\; t \\ \text{such that } \robj^\top \hpnash_\samp e_i \geq t \;
  \text{for all } i \in [d],\\ \sum_i \robj_i = 1 \quad \text {and}
  \quad \robj_i \geq 0 \; \text{for all } i \in [d].
  \end{gathered}
\end{equation}
Before stating our main result concerning the asymptotic distribution
of the error $\DeltaA(\pihatplug, \robjs)$, we introduce some notation
first. Let us denote by $\vars = (\robj, t)$ the variables and by
matrix $\cmat$ and vector $\csim$ the set of constraints in the linear
program~\eqref{eq:nash-lp}, that is,
\begin{align}
  \cmat \defn \begin{bmatrix} \pnash^\top & -\ones_\adim\\ I_\adim &
    \zeros
\end{bmatrix} \quad \text{and} \quad
\csim \defn [\ones_\adim, \zeros],
  \end{align}
where we have denoted by $\ones_\adim$ the all-ones column vector in
$\adim$ dimension and by $I_\adim$ the $\adim \times \adim$ identity
matrix. Using this notation, we can rewrite this LP as
\begin{align}
  \begin{gathered}
  \max\; t \\ \text{such that } \cmat \vars \geq \zeros,\;\;
  \csim^\top \vars = \ones
  \end{gathered}
\end{align}
It will also be convenient to define the extended matrix $\cext \defn
[\cmat;\csim^\top]$ which contains both the equalit and inequality
constraints. For the perturbed version of the linear
program~\eqref{eq:nash-lp-pert} we denote the analagous matrices
respectively by $\cmathat$ and $\cexthat$. Observe that the simplex
constraint encoded by the vector $\csim$ is deterministic and hence
remains the same for both the original and perturbed linear programs.

Observe that the constraint polytope for the LP~\eqref{eq:nash-lp} is
a closed convex set since the Nash distribution $\robj$ belongs to the
simplex $\simplex_\adim$ and the variable $t \in [0,1]$. Therefore,
the optimal solution $\vars^* = (\robjs, t^*)$ will lie on one of
faces whose dimension $0 \leq \dface \leq \adim$. In the special case
when $\dface = 0$, we say that the LP admits a unique solution which
is a vertex of the constraint polytope. Let us denote by subsets
$\conset_1 \subseteq \{1, \ldots, \adim\}$ and $\conset_2\subseteq
\{\adim+1, \ldots, 2\adim \}$ the subset of constraints (rows of the
constraint matrix $\cmat$) which are tight for the set of optimal
solutions and let us represent their union by \mbox{$\conset =
  \conset_1 \cup \conset_2$}. Observe that in addition to the equality
constraint $\csim^\top \vars = 1$, there can be at most $\adim$
constraints tight, that is, $|\conset| \leq \adim$. Further, we denote
by $\consethat_1, \consethat_2$ and $\consethat$ the corresponding
subsets for the perturbed linear program~\eqref{eq:nash-lp-pert}.

We first establish a technical lemma which establishes that given
enough samples, the active constraints for the original
LP~\eqref{eq:nash-lp} given by $\conset$ will be contained in the
active constraints $\consethat$ for the solution of the perturbed
LP~\eqref{eq:nash-lp-pert}.
\begin{lemma}\label{lem:subset-pert}
Consider the perturbed LP~\eqref{eq:nash-lp-pert} for any payoff
matrix $\pnash \in [0,1]^{\adim \times \adim}$ with noise distribution
following the Gaussian passive sampling model. Then, for all $\samp >
\samp_0(\pnash, \delta)$, we have that the active constraint sets
$\conset$ for the original LP and $\consethat$ for the perturbed LP
satisfy $\conset \subseteq \consethat$ with probability at least
$1-\delta$.
\end{lemma}
We defer the proof of the lemma to the end of the section. Observe
that depending on the sampling of the noise variables, the subset
$\consethat$ can vary with the noise variables. Each of these
different subset can be seen as adding additional constraints on top
of the $|\conset|$ constraints which characterize the set of Nash
equilibria for the original LP. Thus, when we look at the constraint
matrix $\cextj{\consethat}$, any $\vars = (\robj, t)$ satisfying
$\cextj{\consethat}\cdot\vars = [\zeros_\adim, \ones]^\top$ will
necessarily have $\robj$ as a Nash equilibrium.

Before stating our main result, we introduce some notation which is
essential for the statement. Let use represent by $\Aalt \defn
\pnash^\top_{\consethat_1, \consethat_2^c}$ the rank $r$ matrix of
constraints which are tight in the perturned LP and its singular value
decomposition by $\Aalt = \Us \Sigs \Vs^\top$ and the corresponding
noisy matrix
\begin{align*}
\hpnash^\top_{\consethat_1, \consethat_2^c} = \begin{bmatrix} \Us_1&
  \Us_2
\end{bmatrix}
\begin{bmatrix}
  \Sigs_1 & 0\\
  0& 0
\end{bmatrix}
\begin{bmatrix}
  \Vs_1^\top \\ \Vs_2^\top
\end{bmatrix}
+ \underbrace{\begin{bmatrix}
Z_{11} & Z_{12}\\
Z_{21} & Z_{22}
\end{bmatrix}}_{Z_n}\;,
\end{align*}
where the matrix $Z$ represents average zero-mean gaussian noise with
$\samp$ samples obtained from the passive sampling model. Further, we
let $\Ztnoise_\samp \defn \Us^\top \Znoise_\samp\Vs$ denote the noise
matrix rotated by the directions given in $\Us$ and $\Vs$. With this,
we state our result which characterizes the error $\DeltaA$ for the
plug-in estimator $\pihatplug$ in terms of properties of the
underlying matrix $\pnash$ and the noise matrix $\Ztnoise$.
\begin{theorem}\label{thm:local-nash}
  For any payoff matrix $\pnash \in [0,1]^{\adim \times \adim}$ and confidence $\delta \in (0,1)$, there exists a constant $\samp_0(\pnash, \delta)$ such that
  for samples $\samp > \samp_0(\pnash, \delta)$ obtained via the
  Gaussian passive sampling model, the error $\DeltaA$ of the
  plug-in estimate $\pihatplug$ satisfies
\begin{align}
  \DeltaA(\pihatplug, \robjs) &\leq t^*\max_{i \in
    \consethat_1} e_i^\top \Us_1\left(\Ztnoise_{11} -
  \Ztnoise_{12}\Ztnoise_{22}^{-1}\Ztnoise_{21}
  \right)\Sigs_{1}^{-1}\Us_1^\top \ones_{\consethat_1} + c\sqrt{d}\cdot \|\Sigma_1^{-1} \|_2^2\cdot \|\Ztnoise_{12}\Ztnoise_{22}^{-1}\Ztnoise_{21}
  -\Ztnoise_{11} \|_2^2 \nonumber\\
  &\quad + c\sqrt{d}(\hat{t} -
  t^*)(1+\|\left(\Ztnoise_{11} -
\Ztnoise_{12}\Ztnoise_{22}^{-1}\Ztnoise_{21}
\right)\Sigs_{1}^{-1} \|_2)
\end{align}
with probability at least $1-\delta$ for some universal constant $c>0$.
\end{theorem}
A few comments on the theorem are in order. Observe that the upper
bound on the error is a stochastic quantity where the randomness is
not only in the entries of the matrix $\tilde{Z}$ but also in the
matrices $\Us$ and $\Sigs$ which depend on the (possibly) random
subsets $\consethat_1$ and $\consethat_2$. The upper bound depends
primarily on two terms, up to lower order error factors, one measures
the alignment of the Schur complement $\Ztnoise_{11} -
\Ztnoise_{12}\Ztnoise_{22}^{-1}\Ztnoise_{21}$ with the rotated and
renormalized ones vector $\ones_{\consethat_1}$, and the second which
measures the convergence of the empirical value $\hat{t}$ to the true
value $t^*$. Going forward, we first provide a complete proof this
result and then specialize it to the special case when the true Nash
$\robjs$ is unique and lies in the interior of the simplex
$\simplex_\adim$ -- this greatly simplifies the above expression and
allows us to study the problem dependent adaptivity properties of the
plug-in estimator $\pihatplug$. Additionally, the probability of error $\delta$ and the corresponding restriction on sample size in the above statement come from conditioning on the event $\{J \subseteq \hat{J}\}$ (from Lemma~\ref{lem:subset-pert}).

\begin{remark}
  The above analysis can be extended to the multi-criteria preference
  learning setup $\cdim > 1$ whenever the distance function $\distf =
  \norm_\infty$ and the target set $\Sc$ is a polytope by extending
  the linear program to handle the additional constraints. Similar to
  the result above, the final upper bound on the error will then
  depend only on the constraints which are tight in the original and
  perturbed programs. It remains an interesting problem to study the
  asymptotic error for general convex target sets for which the
  optimization problem can be written as a convex program.
\end{remark}

\begin{proof}[Proof of Theorem~\ref{thm:local-nash}]
We will establish the claim by analyzing the structure of the solution
$\hat{\vars} = (\pihatplug, \hat{t})$ output by solving the perturbed
linear program~\eqref{eq:nash-lp-pert}. Recall from our notation that
given access to $\samp$ noisy samples of the matrix $\pnash$, the set
$\consethat = \consethat_1 \cup \consethat_2$ represents the set of
constraints which are tight for the perturbed LP with the empirical
matrix $\hpnash$ where $|\consethat| = d$. Also, observe that since
these samples come from the Gaussian passive sampling model, we will
have that the solution $\hat{\vars}$ will be unique\footnote{For the
case when an entire column of matrix $\pnash$ is determinisitc, those
constraints (if tight) can be combined with the other deterministic
constraints and the analysis can proceed from there.} with
probability~$1$.

Given this uniqueness, we can express the solution $\hat{\vars} =
(\hat{\robj}, \hat{t})$ as the solution to the linear system
\begin{align*}
\begin{bmatrix}
  \hpnash^\top_{\consethat_1} &
  -\ones_{|\consethat_1|}\\ I_{\consethat_2} &
  \zeros_{|\consethat_2|}\\ \ones_{\adim} & \zeros
\end{bmatrix}\cdot
\begin{bmatrix}
  \pihat_{|\consethat_1|}\\ \pihat_{|\consethat_1|}\\ \hat{t}
\end{bmatrix} =
\begin{bmatrix}
  \zeros_{|\consethat_1|}\\ \zeros_{|\consethat_2|}\\ \ones
\end{bmatrix}\;.
\end{align*}
Let us denote by the vector $\bvec{\consethat} \defn
[-\ones_{|\consethat_1|}, \zeros_{|\consethat_2|}]^\top$ and by the
matrix $\tilde{\cmat}_{\consethat} \defn [\hpnash^\top_{\consethat_1};
  I_{\consethat_2}]$. Using a standard block matrix inversion formula,
we have that the output solution
\begin{align*}
  \pihat = \hat{t} \tilde{\cmat}_{\consethat}^{-1}\bvec{\consethat}
  \quad \text{and} \quad \hat{t} = \frac{1}{\ones_d^\top
    \tilde{\cmat}_{\consethat}^{-1}\bvec{\consethat}}.
\end{align*}
In order to further simplify the above expression, let us denote by
$\hpnash_{\consethat_1, \consethat_2^c}$ the matrix formed by
selecting the rows $\consethat_1$ and the columns $\consethat_2^c
\defn [\adim]\setminus\consethat_2$ from the matrix $\hpnash$. The
estimate $\pihat$ is then given by
\begin{align*}
\pihat_{\consethat_2^c} = -\hat{t}\hpnash_{\consethat_1,
  \consethat_2^c}^{-T}\cdot\ones_{|\consethat_1|}\quad \text{and}\quad
\pihat_{\consethat_2} = 0.
\end{align*}
Plugging in the above value of the estimate $\pihat$ into the error
term $\DeltaA(\pihat, \pi^*)$, we obtain
\begin{align*}
  \DeltaA(\pihat, \robjs) &= \min_{i \in [d]} e_i^\top \pnash^\top
  \robjs - \min_{i' \in [d]} e_{i'}^\top \pnash^\top
  \pihat\\ &\stackrel{\1}{=} \max_{i' \in [d]} \min_{i \in [d]}
  e_i^\top \pnash_{\cdot, \consethat_2^c,
  }^\top\robjs_{\consethat_2^c} - e_{i'}^\top \pnash_{\cdot,
    \consethat_2^c }^\top
  \pihat_{\consethat_2^c}\\ &\stackrel{\2}{\leq} \max_{i \in
    \consethat_1} e_i^\top \pnash_{\consethat_1, \consethat_2^c
  }^\top\left( \robjs_{\consethat^c_2} -
  \pihat_{\consethat^c_2}\right)\\ &\stackrel{\3}{=} \max_{i \in
    \consethat_1} e_i^\top \pnash_{ \consethat_1,
    \consethat_2^c}^\top\left( \hat{t}\hpnash^{-T}_{\consethat_1,
    \consethat_2^c} -t^*(\pnash^\top_{\consethat_1,
    \consethat_2^c})^\pinv\right)\ones_{|\consethat_1|}
\end{align*}
where equality $\1$ follows from noting that one of the Nash
equilibria will have the components $\robjs_{\consethat_2} = 0$ from
Lemma~\ref{lem:subset-pert}, $\2$ follows from upper bounding the min
and noting that the only columns of $\pnash$ that can be minimizers
are those in $\consethat_1$ for large enough samples $\samp$, and $\3$
follows by substituting the values of $\pihat$ and the nash
distribution $\robjs$. We can further split the error term into two
components, one which looks at the error in value $\hat{t}$, and the
other corresponding to the error in matrix $\hpnash$.
\begin{align}
\label{eq:err-decom-local}
\DeltaA(\pihat, \robjs) &\leq t^*\max_{i \in \consethat_1} e_i^\top
\pnash_{ \consethat_1, \consethat_2^c}^\top\left(
\hpnash^{-\top}_{\consethat_1, \consethat_2^c}
-(\pnash^\top_{\consethat_1,
  \consethat_2^c})^\pinv\right)\ones_{|\consethat_1|} + (\hat{t} -
t^*)\max_{i \in \consethat_1} e_i^\top \pnash_{ \consethat_1,
  \consethat_2^c}^\top \hpnash^{-\top}_{\consethat_1,
  \consethat_2^c}\ones_{|\consethat_1|}
\end{align}
Let us denote by $\Aalt \defn \pnash^\top_{\consethat_1,
  \consethat_2^c}$. Then, we can rewrite the matrix $
\hpnash^{\top}_{\consethat_1, \consethat_2^c} = \Aalt + \Znoise_\samp$
where the matrix $\Znoise_\samp$ represents the zero-mean noise from
the Gaussian passive sampling model. Further, let $\Phi =
\Us\Sigs\Vs^\top$ denote the SVD of the matrix $\Aalt$. With this, the
first term in the above decomposition for any fixed value of $i$ is
given by
\begin{align*}
e_i^\top \Aalt((\Aalt + \Znoise_\samp)^{-1} -
\Aalt^\pinv)\ones_{|\consethat_1|} = e_i^\top \Us
\Sigs\left((\Sigs+\Us^\top\Znoise_\samp\Vs)^{-1} -
\Sigs^\pinv\right)\Us^\top \ones_{|\consethat_1|}.
\end{align*}
Let us denote by $\Ztnoise_\samp \defn \Us^\top \Znoise_\samp\Vs$ the
effective noise matrix. Using the block matrix inversion formula, the
above expression can be written as
\begin{align}
  e_i^\top \Aalt((\Aalt + \Znoise_\samp)^{-1} -
  \Aalt^\pinv)\ones_{|\consethat_1|} &= e_i^\top\begin{bmatrix} \Us_1&
  \Us_2
  \end{bmatrix}
  \begin{bmatrix}
    \Sigs_1 & 0\\ 0& 0
  \end{bmatrix}
  \left(
  \begin{bmatrix}
    \Sigs_1 + \Ztnoise_{11} & \Ztnoise_{12}\\ \Ztnoise_{21} &
    \Ztnoise_{22}
  \end{bmatrix}^{-1} - \begin{bmatrix}
  \Sigs_1^{-1}& 0 \nonumber\\ 0 & 0
\end{bmatrix}
  \right)
  \begin{bmatrix}
    \Us_1^\top\\ \Us_2^\top
  \end{bmatrix}
  \ones_{\consethat_1}\\ &\stackrel{\1}{=} e_i^\top \Us_1
  \Sigs_1\left((\Sigs_1 + \Ztnoise_{11} -
  \Ztnoise_{12}\Ztnoise_{22}^{-1}\Ztnoise_{21} )^{-1} - \Sigs_1^{-1}
  \right)\Us_1^\top \ones_{\consethat_1}.
\end{align}
where $\Sigs_1$ is the diagonal matrix with non-zero singular value of
$\Aalt$ and equality $\1$ follows from the fact that $\Us_2^\top
\ones_{\consethat_1} = 0$. To see this, recall that $\Us$ represents
the column space of the matrix $\Aalt = \pnash^\top_{\consethat_1,
  \consethat_2^c}$, that is the row space of the matrix
$\pnash^\top_{\consethat_1, \consethat_2^c}$ with $\Us_2$ representing
the null space of this matrix. Since we know that $(\robjs)^\top
\pnash^\top_{\consethat_1, \consethat_2^c} =
t^*\ones_{|\consethat_1|}$, all vectors in the null space will
necessarily have to be orthogonal to the vector
$\ones_{|\consethat_1|}$. Combining the above error bound with
equation~\eqref{eq:err-decom-local}, we find that
\begin{align*}
  \DeltaA(\pihat, \robjs) &\leq t^*\max_{i \in \consethat_1} e_i^\top
  \Us_1 \Sigs_1\left((\Sigs_1 + \Ztnoise_{11} -
  \Ztnoise_{12}\Ztnoise_{22}^{-1}\Ztnoise_{21} )^{-1} - \Sigs_1^{-1}
  \right)\Us_1^\top \ones_{\consethat_1}\\
  &\quad + (\hat{t} -
  t^*)\max_{i \in \consethat_1} e_i^\top \Us_1 \Sigs_1\left((\Sigs_1 +
  \Ztnoise_{11} - \Ztnoise_{12}\Ztnoise_{22}^{-1}\Ztnoise_{21} )^{-1}
  \right)\Us_1^\top \ones_{\consethat_1}\\
  &\stackrel{\1}{\leq} t^*\max_{i \in
    \consethat_1} e_i^\top \Us_1\left(\Ztnoise_{11} -
  \Ztnoise_{12}\Ztnoise_{22}^{-1}\Ztnoise_{21}
  \right)\Sigs_{1}^{-1}\Us_1^\top \ones_{\consethat_1} \\
  &\quad +
  t^* \max_{i \in
    \consethat_1} e_i^\top \Us_1\left(\sum_{s=2}^\infty  \left(\left(
  \Ztnoise_{12}\Ztnoise_{22}^{-1}\Ztnoise_{21}
  -\Ztnoise_{11} \right)\Sigma_1^{-1}\right)^s\right)\Us_1^\top \ones_{\consethat_1}\\
  &\quad + (\hat{t} -
  t^*)\max_{i \in \consethat_1} e_i^\top \Us_1 \Sigs_1\left((\Sigs_1 +
  \Ztnoise_{11} - \Ztnoise_{12}\Ztnoise_{22}^{-1}\Ztnoise_{21} )^{-1}
  \right)\Us_1^\top \ones_{\consethat_1}\\
  &{\leq} t^*\max_{i \in
    \consethat_1} e_i^\top \Us_1\left(\Ztnoise_{11} -
  \Ztnoise_{12}\Ztnoise_{22}^{-1}\Ztnoise_{21}
  \right)\Sigs_{1}^{-1}\Us_1^\top \ones_{\consethat_1} + c\sqrt{d}\cdot \|\Sigma_1^{-1} \|_2^2\cdot \|\Ztnoise_{12}\Ztnoise_{22}^{-1}\Ztnoise_{21}
  -\Ztnoise_{11} \|_2^2 \\
  &\quad + (\hat{t} -
  t^*)\max_{i \in \consethat_1} e_i^\top \Us_1 \Sigs_1\left((\Sigs_1 +
  \Ztnoise_{11} - \Ztnoise_{12}\Ztnoise_{22}^{-1}\Ztnoise_{21} )^{-1}
  \right)\Us_1^\top \ones_{\consethat_1}\\
  &\stackrel{\2}{\leq} t^*\max_{i \in
    \consethat_1} e_i^\top \Us_1\left(\Ztnoise_{11} -
  \Ztnoise_{12}\Ztnoise_{22}^{-1}\Ztnoise_{21}
  \right)\Sigs_{1}^{-1}\Us_1^\top \ones_{\consethat_1} + c\sqrt{d}\cdot \|\Sigma_1^{-1} \|_2^2\cdot \|\Ztnoise_{12}\Ztnoise_{22}^{-1}\Ztnoise_{21}
  -\Ztnoise_{11} \|_2^2 \\
  &\quad + c\sqrt{d}(\hat{t} -
  t^*)(1+\|\left(\Ztnoise_{11} -
\Ztnoise_{12}\Ztnoise_{22}^{-1}\Ztnoise_{21}
\right)\Sigs_{1}^{-1} \|_2)
\end{align*}
where inequality $\1$ follows from the Taylor series expansion $(I-X)^{-1} = \sum_{s=0}^\infty X^s$ and holds whenever $\|X\|_2\leq 1$, which can be established for $n$ large enough, and inequality $\2$ follows from another such Taylor series expansion followed by a
Cauchy–Schwarz inequality. This establishes the desired claim.
\end{proof}

\paragraph{Asymptotic error under uniqueness assumption.}

Having established an upper bound on the error for the general setup
in Theorem~\ref{thm:local-nash}, we now consider the specific scenario
where the payoff matrix $\pnash$ is a preference matrix and has a
unique von Neumann winner $\robjs$. This is formalized in the
following assumption.
\begin{assumption}[Unique Nash equilibrium]
  \label{ass:unique-nash}
The payoff matrix $\pnash$ belongs to the set of preference matrices
$\prefS_{\adim,1}$ and has a unique mixed Nash equilibrium $\robjs$,
that is, $\robjs_i > 0$ for all $i \in [d]$.
\end{assumption}

For any preference matrix $\pnash \in \prefS_{\adim, 1}$ and the
Bernoulli passive sampling model discussed in Section~\ref{sec:main},
the asymptotic variance for the Gaussian passive sampling model is
$\sig_{i,j}^2 = \pnash_{i,j}\cdot(1-\pnash_{i,j})$. Let us represent
by $\Sigma_i$ the diagonal matrix corresponding to the variances along
the $i^{th}$ column of the matrix $\pnash$ with
\begin{align*}
  \Sigma_i = \text{diag}(\pnash_{1,i}\cdot(1-\pnash_{1,i}), \ldots,
  \pnash_{\adim,i}\cdot(1-\pnash_{\adim,i}).
\end{align*}
Given this notation, we now state a corollary which specializes the
result of Theorem~\ref{thm:local-nash} to payoff matrices satisfying
the above assumption.

\begin{corollary}
\label{cor:local-mixed}
For any payoff matrix $\pnash$ satisfying
Assumption~\ref{ass:unique-nash} and confidence $\delta \in (0,1)$, there exists a constant $\samp_0(\pnash, \delta)$ such that for all samples $\samp >\samp_0(\pnash, \delta)$, we have that the error $\DeltaA$ of the
plug-in estimate $\pihatplug$ satisfies
\begin{align}
\DeltaA(\pihatplug, \robjs) &\leq \|\Znoise_\samp \robjs\|_\infty +
O_d(\|\Znoise_\samp\|_2^2)\nonumber\\ &\leq
c\cdot\sqrt{\frac{\sig_{\pnash}^2\adim^2}{\samp}\log\left(
  \frac{\adim}{\delta}\right)} + O_d(\|\Znoise_\samp\|_2^2)\;,
\end{align}
with probability at least $1-\delta$ and the variance $\sig_{\pnash}^2
\defn \max_{i \in [d]} (\robjs)^{\top}\Sigma_i \robjs$.
\end{corollary}

We make a few remarks on the above corollary. Observe that the above
is a high probability bound on the error $\DeltaA$ of the plug-in
estimator $\pihatplug$. Compared with the upper bounds of
Corollaries~\ref{cor:linfty} and~\ref{cor:lone}, the asymptotic bound
on the error above is instance dependent -- the effective variance
$\sig_{\pnash}^2$ depends on the underlying preference matrix
$\pnash$. In particular, this variance measures how well does the
underlying von Neumann winner $\robjs$ align with each variance
associated with each column of the matrix $\pnash$. In the worst case,
since each entry of $\pnash$ is bounded above by $1$, the variance
$\sigma_\pnash^2 = 1$ and we recover back the upper bounds from
Corollaries~\ref{cor:linfty} and~\ref{cor:lone} for the uni-criterion
case. The second term in the upper bound comprising the operator norm
of the sampling noise, $\|\Znoise_\samp\|_2^2$, can be shown to be
$O_{d}(\frac{1}{\samp})$ with high probability\footnote{$O_d$ notation hides the dependence on the dimensionality $d$.}, and therefore
contributes as a lower order term.

\begin{proof}[Proof of Corollary~\ref{cor:local-mixed}]
Observe that Assumption~\ref{ass:unique-nash} implies that the set of
tight constraints for the LP~\eqref{eq:nash-lp} are the ones
corresponding to payoff matrix $A$. That is, the set $\conset_1 = [d]$
and $\conset_2 = \phi$. Following Lemma~\ref{lem:subset-pert}, we
have, for $\samp$ large enough, the subset of tight constraints for
the perturbed LP~\eqref{eq:nash-lp-pert} satisfy $\consethat_1 =
\conset_1$ and $\consethat_2 = \conset_2$. Further, the uniqueness
assumption guarantees that the matrix $\pnash$ is full rank and hence,
invertible.

Since the matrix $\hpnash$ is itself a preference matrix (by
construction), the value $\hat{t} = t^* = \frac{1}{2}$ and therefore,
using the upper bound on the error from Theorem~\ref{thm:local-nash},
we have,
\begin{align*}
  \DeltaA(\pihat, \robjs) &\leq t^* \max_{i \in [\adim]}\left[e_i^\top
    \Znoise_\samp\pnash^{-\top}\ones_{\adim} \right] +
  O_d(\|\Znoise_\samp\|_2^2)\\ &\leq \|\Znoise_\samp \robjs\|_\infty +
  O_d(\|\Znoise_\samp\|_2^2)\;
\end{align*}
where the final inequality follows by noting that $\robjs =
t^*\pnash^{-\top}\ones_{\adim}$ and recall that $\Znoise_\samp$
denotes the zero mean noise-matrix obtained by the passive Gaussian
sampling model with (asymptotic) variance \mbox{$\sigma_{i,j}^2 =
  \pnash_{i,j}\cdot (1-\pnash_{i,j})$}.  Let us denote by
$\Sigma_\pnash$ the diagonal matrix measuring the alignment of the
Nash equilibrium $\robjs$ with the variance of the $i^{th}$ column of
the underlying matrix $\pnash$, that is,
\begin{align*}
  \Sigma_{\pnash}(i,j) = \begin{cases}
    (\robjs)^\top\Sigma_i\robjs\quad &\text{for } i = j \\ 0 \quad
    &\text{otherwise}
  \end{cases}.
\end{align*}
Following a similar calculation as in the proof of
Corollary~\ref{cor:lone}, we have that each entry of the matrix
$\Znoise_{i,j}$ will have samples $N_{i,j} =
\Theta(\frac{\samp}{\adim^2})$. Combined with a standard bound for the
maximum of sub-Gaussian random variables~\cite{wainwright2019}, we
have for $\samp > \samp_0(\pnash, \delta)$, with probability at least
$1-\delta$,
\begin{align*}
  \DeltaA(\pihat, \robjs) &\leq
  c\cdot\sqrt{\frac{\sig_{\pnash}^2\adim^2}{\samp}\log\left(
    \frac{\adim}{\delta}\right)} + O_d(\|\Znoise_\samp\|_2^2)\;,
\end{align*}
for some universal constant $c>0$ and where the variance
$\sig_{\pnash}^2 = \max_{i \in [d]}\Sigma_{\pnash}(i,i)$.
\end{proof}

\paragraph{Example: Generalized Rock-Papers-Scissor.}

While in the worst-case, the variance determining the sample
complexity of learning Nash from samples is $\sig_{\pnash}^2 =
\Theta(1)$, we will now construct a family of preference matrices
$\pnash^{(\adim)}$, for different values of dimension $\adim$, and
show that $\sig_{\pnash}^2 = O(\frac{1}{d})$. This exhibits that the
plug-in estimator $\pihatplug$ can indeed adapt to the problem
complexity and has a sample complexity of
$\tilde{O}(\frac{\adim}{\epsilon^2})$ for these class of easier
problems compared to the worst-case complexity of
$\tilde{O}(\frac{\adim^2}{\epsilon^2})$.

Our example is a high-dimensional generalization of the classical
Rock-Papers-Scissors (RPS) game. Recall, that the pay-off matrix for
the RPS game is
\begin{align*}
\pnash^{\textsf{RPS}} = \begin{tabular}{c|ccc} & \text{R} & \text{P} &
  \text{S} \\ \hline \text{R} &0.5 &0 &1\\ \text{P} &1 &0.5
  &0\\ \text{S} &0 &1 &0.5
\end{tabular}.
\end{align*}
Observe that the above payoff matrix encodes a deterministic game:
Rock beats Scissor, Scissor beats Paper, and Paper beats Rock. Similar
to this, we define a randomized version of the above RPS game with
payoffs where we allow a small probability $0.25$ with which the
lesser preferred item in a match-up can defeat the other, for example,
Scissor against Rock. Explicitly, such a payoff matrix $\pnash^{(3)}$
is given by
\begin{align*}
  \pnash^{(3)} = \begin{tabular}{c|ccc}
   & \text{R} & \text{P} & \text{S} \\
   \hline
  \text{R} &0.50 &0.25 &0.75\\
  \text{P} &0.75 &0.50 &0.25\\
  \text{S} &0.25 &0.75 &0.50
  \end{tabular}.
\end{align*}
Similar to the deterministic RPS game, the above randomized game can
be seen to have a unique Nash equilibrium with $\robjs = [\frac{1}{3},
  \frac{1}{3}, \frac{1}{3}]$.

We now describe a $\adim$-dimensional generalization of the above
payoff matrix, for any odd value of $d = 2d'+1$. For the element
$e_1$, the first entry will be set to $0.50$, the next $d'$ entries to
be $0.25$ and the final $d'$ entries to be $0.75$ --
game-theoretically, this means that the element $e_1$ loses to
elements $e_{2}-e_{d'+1}$ and is preferred over elements
$e_{d'+2}-e_{2d'+1}$, both with probability $0.75$. Similarly, for the
$i^{th}$ element, the row $\pnash^{(d)}_i$ is given by
\begin{equation*}
  \pnash^{(d)}(i,j) = \begin{cases} 0.50 \quad &\text{for } j =
    i\\ 0.25 \quad &\text{for } j \in [i+1\;(\text{mod } d),
      i+d'\;(\text{mod } d)]\\ 0.75 \quad &\text{for } j \in
    [i+d'+1\;(\text{mod } d), i+2d'\;(\text{mod } d)]
\end{cases}.
\end{equation*}
It is easy to see from the form of the pay-off matrix that each
element $e_i$ is preferred over $d'$ elements and has a lower
preference than $d'$ elements. By the symmetry of the payoff matrix,
the unique Nash equilibrium is given by the distribution $\robjs =
\frac{1}{\adim}\ones_\adim$ which lies in the interior of the simplex
$\simplex_\adim$ and hence satisfies
Assumption~\ref{ass:unique-nash}. Further, we can compute the variance
$\sig_{\pnash^{(d)}}^2$ as
\begin{equation*}
  \sig_{\pnash^{(d)}}^2 = \max_{i \in [d]} (\robjs)^\top \Sigma_i
  \robjs = \max_i \left( \frac{1}{4d^2} + \sum_{j\neq i}
  \frac{3}{16d^2} \right) \leq \frac{1}{d^2} + \frac{3}{16d},
\end{equation*}
Plugging this variance in the upper bound obtained in
Corollary~\ref{cor:local-mixed}, for $\samp > \samp_0(\pnash,
\delta)$, we have
\begin{equation}
  \Delta_{\pnash^{(d)}}(\pihatplug, \robjs) \leq c
  \sqrt{\frac{d}{\samp}\log\left(\frac{d}{\delta} \right)}
\end{equation}
with probability greater than $1-\delta$. Thus, to obtain an
$\epsilon$-accurate solution for the payoff matrix $\pnash^{(d)}$, the
plug-in estimator $\pihatplug$ requires
$\tilde{O}(\frac{\adim}{\eps^2})$ samples, a factor $\adim$ less than
the worst-case sample complexity of
$\tilde{O}(\frac{\adim^2}{\eps^2})$.

\subsection*{Proof of Lemma~\ref{lem:subset-pert}}

Recall from our discussion above that the the constraint set for the
LP~\eqref{eq:nash-lp}, the constraint polytope is a closed convex set
and has a finite number $|V| = O((2\adim+2)^{\frac{\adim+1}{2}})$,
which follows from McMullen's theorem~\cite{mcmullen1970}. Let us
denote each vertex of the polytope by $\vars_v = (\robj_v, t_v)$ and
the corresponding set of $\adim$ constraints which define the vertex
by $\conset_v$. Further, let $V^*$ denote the set of optimal vertices.

Because of the random Gaussian noise, for $\samp = \Omega(\adim^2)$,
we have that the solution $\hat{\vars} = (\pihatplug,\hat{t})$ will be
unique with probability $1$. In order to establish the claim of the
lemma, we can equivalently show that for any vertex $\vars_v \notin
V^*$, the corresponding vertex $\hat{\vars}_v$ will not be output by
the perturbed LP. This follows from the observation that for $\samp >
O(d^2\log(1/\delta))$ and for any vertices $\vars_v$, we have
\begin{align*}
  \prob\left(|\hat{t}_v - t_v| \geq
  c\sqrt{\frac{\adim^2}{\samp}\log\left(\frac{1}{\delta}
    \right)}\right) \leq \delta,
\end{align*}
for some universal constant\footnote{To be clear, the value of the
constant $c$ can change values across lines, but will always remain a
universal constant independent of problem parameters.} $c>0$. Taking a
union bound over all $|V|$ vertices, we have
\begin{align*}
  \prob\left(\exists\; \vars_v\text{ s.t. }|\hat{t}_v - t_v| \geq
  c\sqrt{\frac{\adim^2}{\samp}\log\left(\frac{|V|}{\delta}
    \right)}\right) \leq \delta.
\end{align*}
Therefore, whenever $\max_vt_v \leq t^* - \gamma$ for some $\gamma >
0$, we have that after $\samp > c\frac{\adim^2}{\gamma^2}\log\left(
\frac{|V|}{\delta}\right)$, with probability at least $1-\delta$, we
have,
\begin{align}
  \max_{v \notin V^*} \hat{t}_v < \min_{v \in V^*} \hat{t}_v\;,
\end{align}
and therefore the perturbed LP will have active constraint set
satisfying $\conset \subseteq \consethat$. This proves the desired
claim.

\section{Additional results and their proofs}
\label{app:add_res}

This section covers additional sample complexity results as well as
optimization algorithms for finding the Blackwell winner of a
multi-criteria preference learning instance.

\subsection{Sample complexity bounds for $\ell_1$-norm}

\begin{corollary}
\label{cor:lone}
Suppose that the distance $\distf$ is induced by the $\ell_1$-norm $\|
\cdot \|_1$. Then there are universal constants such that given a
sample size $\samp > c_1 \adim^2\cdim \log(\frac{c_2 \adim
  \cdim}{\conf})$, then for each valid target set $\Sc$, we have
\begin{equation}
\DeltaP(\pihatplug, \pi^*) \leq c_3 \cdim \sqrt{ \frac{\adim^2
    \cdim}{\samp}\log\left(\frac{c_3 \adim \cdim}{\conf}\right)}
\end{equation}
with probability exceeding $1 - \conf$.
\end{corollary}
\begin{proof}
Being somewhat more explicit with our notation, let $N_{(i_1, i_2,
  j)}$ denote the number of samples observed under the passive
sampling model at index $(i_1, i_2, j)$ of the tensor. Proceeding as
in equation~\eqref{eq:Hoeff-cond}, we have
\begin{equation*}
  \Pr \left\{ \|\pref^\jj(\cdot, \ic) - \hpref^\jj(\cdot,
  \ic)\|_\infty \geq c\sqrt{ \frac{\log (c \adim / \delta)}{\min_{i_1
        \in [\adim]} N_{(i_1, i_2, j)}} }\right\} \leq \delta.
\end{equation*}
Summing over all criteria $\jj \in [\cdim]$ and then applying the
union bound yields
\begin{equation*}
\Pr \left\{ \|\pref(\cdot, \ic) - \hpref(\cdot, \ic)) \|_{\infty, 1}
\geq c\cdim \sqrt{ \frac{\log (c \adim \cdim/ \delta)}{\min_{i_1, j}
    N_{(i_1, i_2, j)}} }\right\} \leq \delta.
\end{equation*}
Finally, in order to obtain a bound on the maximum deviation in the
$(\infty, 1)$-norm, we take a union bound over all $\adim$ choices of
the index $i_2$, and apply inequality~\eqref{eq:Ni} to obtain
\begin{equation*}
  \max_{\ic}\|\pref(\cdot, \ic) - \hpref(\cdot, \ic)) \|_{\infty, 1}
  \leq c \cdim \sqrt{\frac{\adim^2\cdim}{\samp} \log\left(c
    \frac{\adim\cdim}{\conf}\right)}
\end{equation*}
with probability exceeding $1 - \delta$.
\end{proof}

A few comments regarding the corollary are in order. The above
corollary suggests that the sample complexity required for obtaining
an $\eps$-accurate solution with respect to the $\ell_1$ norm is
$\samp = \widetilde{O}(\frac{\adim^2\cdim^3}{\eps^2})$. Observe that
this bound is a factor of $\cdim^2$ worse than the corresponding one
for $\ell_\infty$ norm established in Corollary~\ref{cor:linfty}. This
additional sample complexity occurs since for any vector $v \in
\real^\cdim$, we have $\|v\|_1 \leq \cdim \|v \|_\infty$. This implies
that the error when measured with respect to $\ell_1$ can be upto
$\cdim$ times larger; since the sample complexity scales as
$\frac{1}{\eps^2}$, the corresponding increase with respect to the
number of criteria $\cdim$ is quadratic.

\subsection{Optimization algorithms}

Recall that Theorem~\ref{thm:opt} established that the objective
function $\vg(\robj;\pref, \Sc, \norm_q)$ is convex in $\robj$ and
Lipschitz with respect to the $\ell_1$ norm. This implies that one
could compute the plug-in solution $\pihatplug$ as a solution to a
constrained optimization problem. In this section, we discuss a few
specific algorithms based on zeroth-order and first-order methods for
obtaining such a solution.

\subsubsection{Zeroth-order optimization}

\begin{algorithm}[t!]
	\DontPrintSemicolon \KwIn{Time steps $\T$, step size $\step$,
          smoothing radius $\smrad$} \textbf{Initialize:} $\parop_1 =
        0$\; \For{$\ti = 1, \ldots, \T$}{ $\robj_\ti = \arg
          \max_{\robj \in \simplex_\adim} \inner{\parop_\ti}{\robj} -
          \reg(\robj)$ where $\reg(\robj) = \sum_\ii \robj_\ii
          \log(\robj_\ii)$\; Sample $\smvec_\ti$ uniformly from the
          Euclidean unit sphere $\{\smvec\; | \; \| \smvec\|_2 = 1
          \}$\; For every $\ii \in [\adim]$, query points $\z_{1, \ii}
          = \pref(\robj_\ti+ \smrad \smvec_\ti, \ii)$ and $\z_{2, \ii}
          = \pref(\robj_\ti+ \smrad \smvec_\ti, \ii)$\; Set
          $\vg(\robj_\ti+ \smrad \smvec_\ti;\pref, \Sc, \distf) =
          \max_{i}\distf(\z_{1, \ii}, \Sc) $ and $\vg(\robj_\ti-
          \smrad \smvec_\ti;\pref, \Sc, \distf) =
          \max_{i}\distf(\z_{2, \ii}, \Sc) $\; Set sub-gradient
          estimate $\subg_\ti = \frac{\adim}{2\smrad}
          \left(\vg(\robj_\ti + \smrad \smvec_\ti;\pref, \Sc, \distf)
          - \vg(\robj_\ti - \smrad \smvec_t;\pref, \Sc, \distf)
          \right)\smvec_\ti$\; Update $\parop_{\ti+1} = \parop_{\ti} -
          \step \subg_{\ti}$ } \KwOut{$\bar{\robj}_\T =
          \frac{1}{\T}\sum_{\ti = 1}^\T\robj_\ti$}
	\caption{Zeroth-order method for multi-criteria preference learning}
  \label{alg:zopl}
\end{algorithm}

Zeroth-order methods for minimizing a function $\f(\x)$ over $\x \in
\X$ work with a function query oracle. That is, at each time step, the
algorithm has access to an oracle which returns the value $\f(\x)$ for
any point $\x \in \X$. In our setup, since we are interested in
minimizing the value function $\vg(\robj;\pref, \Sc, \distf)$ over
$\robj \in \simplex_\adim$, such a function query requires access to
the target set $\Sc$ via an oracle $\oracleS^0$ such that
\begin{equation*}
 \oracleS^0(\z) \rightarrow \min_{\z_1 \in \Sc}\distf(\z, \z_1) \;,
\end{equation*}
for the underlying distance function $\distf(\cdot)$. The oracle
$\oracleS^0$ essentially takes as input a score vector $\z \in
[0,1]^\cdim$ and outputs the distance of this point to the target set
$\Sc$. Given this oracle, it is easy to see that for any $\robj$, one
can compute the corresponding value function $\vg(\robj;\pref, \Sc,
\distf)$.

There have been several algorithms proposed for optimization with such
oracles when the underlying function $\f$ is convex~\cite{flaxman2005,
  agarwal2010, shamir2013, duchi2015, nesterov2017, shamir2017} or
non-convex, smooth~\cite{ghadimi2013}. The key idea in the proposed
algorithms is to utilize the zeroth-order oracle to construct
estimates of the (sub-)gradient of the function $\f$ using a class of
techniques called \emph{randomized smoothing}. The algorithms then
differ in the construction of these estimates depending on the
underlying randomness as well as on the number of oracle calls during
each time step.

Given the results of Theorem~\ref{thm:opt}, we can restrict our focus
on algorithms for the class of convex Lipschitz function $\f$. To this
end, Shamir~\cite{shamir2017} proposed an algorithm for optimizing
such functions which required \emph{two} function evaluations at each
time. The algorithm, adapted to the multi-criteria preference learning
problem, is detailed in Algorithm~\ref{alg:zopl}. For our setup, we
select the negative entropy regularization, $\reg(\robj) = \sum_\ii
\robj_\ii \log(\robj_\ii)$ to suit the geometry of our domain
\mbox{$\X = \simplex_\adim$}.

At each time step $\ti$, the proposed algorithm maintains an estimate
of the distribution $\robj_\ti$, and queries the function value
$\vg(\cdot;\pref, \Sc, \distf)$ at two points---namely, $\robj_\ti +
\smrad \smvec_t$ and $\robj_\ti - \smrad \smvec_t$---where the random
vector $\smvec$ is sampled uniformly from the Euclidean unit sphere
and $\smrad >0$ represents the smoothing radius. Given these queries,
we compute the sub-gradient estimate
\begin{equation*}
  \subg_\ti \defn \frac{\adim}{2\smrad} \left(\vg(\robj_\ti + \smrad
  \smvec_\ti;\pref, \Sc, \distf) - \vg(\robj_\ti - \smrad
  \smvec_t;\pref, \Sc, \distf) \right)\smvec_\ti\;.
\end{equation*}
This sub-gradient estimate is then used to update the parameter
estimate $\robj_{\ti+1}$ using the mirror descent algorithm with the
specified regularization function. The zeroth-order method in
Algorithm~\ref{alg:zopl} does not require the underlying function to
be smooth and hence works for our problem setup with arbitrary
non-differentiable distance functions. We can now obtain the following
convergence result, based on Theorem~1 from the work of
Shamir~\cite{shamir2017}.

\begin{proposition}
  \label{prop:conv_zero}
Under the conditions of Theorem~\ref{thm:opt}, suppose that we run
Algorithm~\ref{alg:zopl} for $\T$ iterations with step-size $\step_\ti
= \frac{c}{\cdim^{1/q}\sqrt{\adim \T}}$ and smoothing radius $\smrad =
\frac{c\log\adim}{\sqrt{\T}}$.  Then the resulting sequence $\robj_1,
\robj_2, \ldots, \robj_\T$ satisfies
\begin{equation*}
  \vg\left(\bar{\robj}_\T; \pref, \Sc, \| \cdot \|_q \right) \leq
  \min_{\pi \in \simplex_{\adim} } \vg(\pi; \pref, \Sc, \| \cdot \|_q
  ) + c \cdim^{\frac{1}{q}}\cdot \sqrt{\frac{\adim \log^2\adim}{\T} }
\end{equation*}
where $\bar{\robj}_\T = \frac{1}{\T}\sum_{\ti=1}^\T \robj_\ti$.
\end{proposition}
\begin{proof}
By Theorem~\ref{thm:opt}, the value function $\vg(\robj;\pref, \Sc,
\norm_q)$ is convex and $\lipv = \cdim^{\frac{1}{q}}$-Lipschitz with
respect to $\norm_1$. Also, the choice of the regularizer $\reg(\robj)
= \sum_\ii \robj_\ii \log(\robj_\ii)$ is $1$-strongly convex with
respect to the $\norm_1$. Plugging in the above values in Theorem~1
from \cite{shamir2017} establishes the above convergence rate.
\end{proof}
Thus, in order to obtain a distribution $\hat{\robj}$ that is
$\eps$-close to $\robjs$ in function value, we need to run
Algorithm~\ref{alg:zopl} for $\T =
O\left(\frac{\cdim^{\frac{2}{q}}\adim \log^2\adim}{\eps^2} \right)$
iterations. Also, note that each iteration of the algorithm requires
$\adim$ calls to the oracle $\oracleS^0$. Therefore the total oracle
complexity of the procedure is
$O\left(\frac{\cdim^{\frac{2}{q}}\adim^2 \log^2\adim}{\eps^2}
\right)$.

\subsection{First-order optimization}

In this section, consider some first-order methods to compute the
plug-in estimator.  
%
Let us denote by $\subd \vg(\robj)$ the set of sub-differentials of
the function $\vg(\cdot;\pref, \Sc, \norm)$ evaluated at
$\robj$. Further, let the set $\maxset(\robj)$ denote the set of
maximizers for a policy $\robj$, that is,
\begin{equation}
\label{eq:maxset}
  \maxset(\robj) = \left\lbrace\tilde{\robj} \in \simplex_\adim\; | \;
  \tilde{\robj} \in \arg \max_{\robj_2 \in \simplex_\adim} \pmin_{\z
    \in \Sc} \left[\|\pref(\robj, \robj_2) - \z \|
    \right]\right\rbrace\;.
\end{equation}
Note that both of these quantities depend implicitly on the tuple
$(\Sc, \pref, \| \cdot \|)$, but we have dropped this dependence in
the notation.  Given the setup above, Lemma~\ref{lem:subg_vg} below
characterizes this set $\subd \vg(\robj)$ for any smooth $\ell_q$ norm
(with $1 <q < \infty$).

\begin{lemma} \label{lem:subg_vg}
Suppose that the distance is induced by a smooth $\ell_q$ norm for $1
< q < \infty$. Then the set of sub-differentials of $\vg$ at $\robj$
is given by:
  \begin{equation*}
    \subd\vg(\robj) = \text{conv}\left\lbrace \frac{\pref(\cdot,
      \robj_2) \left[ \pref(\robj, \robj_2) - \proj_\Sc(\pref(\robj,
        \robj_2))\right]}{\|\pref(\robj, \robj_2) -
      \proj_\Sc(\pref(\robj, \robj_2)) \|_q}\; \vert \; \robj_2 \in
    \maxset(\robj) \right\rbrace\;,
  \end{equation*}
where $\proj_\Sc(\z)$ denotes the unique projection of the point $\z$
onto set $\Sc$ along $\norm_q$.
\end{lemma}

\begin{algorithm}[t!]
	\DontPrintSemicolon \KwIn{Time steps $\T$, step size $\step$}
        \textbf{Initialize:} $\parop_1 = \ones_\cdim$\; \For{$\ti = 1,
          \ldots, \T$}{ Set the distribution $\robj_\ti =
          \frac{\parop_\ti}{\|\parop_\ti\|_1}$\; Obtain $\sg_\ti \in
          \text{conv}\left\lbrace \frac{\pref(\cdot, \robj_2) \left[
              \pref(\robj_\ti, \robj_2) - \proj_\Sc(\pref(\robj_\ti,
              \robj_2))\right]}{\|\pref(\robj_\ti, \robj_2) -
            \proj_\Sc(\pref(\robj_\ti, \robj_2)) \|_q}\; \vert \;
          \robj_2 \in \maxset(\robj_\ti) \right\rbrace$ \hfill[See
            eq.\eqref{eq:maxset} for $\maxset(\robj_\ti)$] \; Update
          $\parop_{\ti+1, \ii} = \robj_{\ti, \ii}\exp(-\step \sg_{\ti,
            \ii}) $ } \KwOut{$\bar{\robj}_\T = \frac{1}{\T}\sum_{\ti =
            1}^\T\robj_\ti$}
	\caption{First-order method for multi-criteria preference learning}
  \label{alg:fopl}
\end{algorithm}

We defer the proof of the above lemma to later in the section. Note that in order to access such a sub-gradient, we need access to an oracle $\oracleS^1$ that provides projection queries of the form
\begin{equation*}
  \oracleS^1(\z) \rightarrow \arg \min_{\z_1 \in \Sc} \distf(\z, \z_1).
\end{equation*}
The oracle $\oracleS^1$ takes in a point $\z$ and outputs the closest point in the set $\Sc$ to this point. Given such an oracle, we can compute the sub-gradient of the function $\vg(\robj;\pref, \Sc, \distf)$ using Lemma~\ref{lem:subg_vg} by evaluating it at the point given by $\pref(\robj, \robj_2)$ for some $\robj_2 \in \maxset(\robj)$.

Given access to such a projection oracle $\oracleS^1$, we detail out a procedure based on a standard implementation of mirror descent with entropic regularization (or Exponentiated gradient method) in Algorithm~\ref{alg:fopl} to minimize the objective $\vg(\robj;\G)$. Note that we select the negative entropy function, $r(\robj) = \sum_\ii \robj_\ii \log(\robj_\ii)$, as the regularization function for the mirror descent procedure since our parameter space is given by the simplex $\simplex_\cdim$ and the negative entropy function is known to be 1-strongly convex with respect to $\norm_1$ over this space.

The algorithm works by maintaining at each time instance a distribution $\robj_\ti$ over the set of objects and updates it via an exponentiated gradient update. That is, the sub-gradient $\sg_\ti$ is evaluated at the current point $\robj_\ti$ using access to both $\oracleS^1$ and $\oracleS^0$, and is used to update each coordinate of the variable $\parop_\ti$. The updated distribution $\robj_{\ti+1}$ is obtained via a KL-projection of $\parop_\ti$ onto the simplex $\simplex_\cdim$, which can be shown to be equivalent to the normalization $\parop/\|\parop \|_1$. We now proceed to prove a convergence result for this gradient-based Algorithm~\ref{alg:fopl}, based on a standard analysis of the mirror descent procedure (for example, see \cite[Theorem 4.2]{bubeck2015}).

\begin{proposition} \label{prop:conv_fopl}
Suppose the conditions of Theorem~\ref{thm:opt} hold and consider any $\ell_q$-norm for $1 < q < \infty$. Suppose that running Algorithm~\ref{alg:zopl} for $\T$ iterations with step-size $\step_\ti = \frac{1}{\cdim^{1/q}}\sqrt{\frac{2\log \adim}{\T}}$ produces a sequence $\robj_1, \robj_2, \ldots, \robj_\T$. Then we have
  \begin{equation*}
    \vg (\bar{\robj}_\T; \pref, \Sc, \| \cdot \|_q) \leq \min_{\robj \in \simplex_{\adim}} \vg(\robj; \pref, \Sc, \| \cdot \|_q) + \cdim^{\frac{1}{q}}\cdot \sqrt{\frac{2\log\adim}{\T} }\;
  \end{equation*}
  where $\bar{\robj}_\T = \frac{1}{\T}\sum_{\ti=1}^\T \robj_\ti$.
\end{proposition}
\begin{proof}
  Note that the function $\vg(\robj;\pref, \Sc, \norm_q)$ is convex and $\cdim^{\frac{1}{q}}$-Lipschitz with respect to the $\ell_1$ norm from Theorem~\ref{thm:opt}. Further, the mirror map given by negative entropy function is 1-strongly convex with respect to $\norm_1$. Plugging in these values in Theorem 4.2 from \cite{bubeck2015}  establishes the required convergence rate.
\end{proof}

In order to obtain an $\eps$-accurate solution in function value, it suffices to run the above algorithm for $\T =O\left(\frac{\cdim^{\frac{2}{q}}\log \adim}{\eps^2}\right)$ iterations, with each iteration using $1$ call to the oracle $\oracleS^1$ and $\adim$ calls to the oracle $\oracleS^0$ (to obtain the set $\maxset$). Thus, we see that the total oracle complexity changes as $\oracleS^1: O\left(\frac{\cdim^{\frac{2}{q}}\log \adim}{\eps^2}\right)$ calls and \mbox{$\oracleS^0: O\left(\frac{\cdim^{\frac{2}{q}}\adim\log \adim}{\eps^2}\right)$} calls -- effectively, an $O({\adim \log \adim})$ decrease in the calls to $\oracleS^0$ is compensated by a corresponding increase of $O(\frac{\log\adim}{\eps^2})$ calls to the stronger oracle $\oracleS^1$.

\paragraph{Proof of Lemma~\ref{lem:subg_vg}.}
  Consider the function $\phi(\robj_1, \robj_2) = \pmax_{\z \in \Sc} \|\pref(\robj_1, \robj_2)  - z \|$ over the domain $\robj_2 \in \simplex_\adim$. For any fixed $\robj_2$, we have that the function $\phi(\robj_1, \robj_2)$ is convex in $\robj_1$. Thus, by Danskin's theorem, we have that the subdifferential set is given by:
  \begin{equation}\label{eq:dans_pd}
    \subd\vg(\robj) = \text{conv}\left\lbrace \frac{\subd\phi(\robj, \robj_2)}{\subd\robj}\; \vert \; \robj_2 \in \maxset(\robj) \right\rbrace\;,
  \end{equation}
where $\text{conv}$ represents the convex hull of the set. Let us now focus on the partial derivative $\frac{\subd\phi(\robj, \robj_2)}{\subd\robj}$ for any $\robj_2$ which is a maximizer. This partial derivative involves differentiation of a metric projection onto a convex set, which has been studied extensively in the literature of convex analysis~\cite{penot1970, zajivcek1984, alimov2014}. Recently, Balestro et al.~\cite{balestro2019} established that for distance functions given by smooth norms, the derivative of metric projection for any $\z \notin \Sc$ is given by:
  \begin{equation*}
    \nabla \distf(\z, \Sc) = \nabla \min_{\z_2 \in \Sc} \|\z - \z_2 \| = \frac{\z -  \proj_\Sc(\z) }{\|\z - \proj_\Sc(\z) \|}\;,
  \end{equation*}
  where $\proj_\Sc(\z)$ denotes the unique projection of the point $\z$ onto set $\Sc$. Combining this with the chain rule of differentiation, we have that:
  \begin{equation*}
    \frac{\subd\phi(\robj, \robj_2)}{\subd\robj} = \frac{\pref(\cdot, \robj_2) \left[ \pref(\robj, \robj_2) - \proj_\Sc(\pref(\robj, \robj_2))\right]}{\|\pref(\robj, \robj_2) - \proj_\Sc(\pref(\robj, \robj_2)) \|_q}\;.
  \end{equation*}
  The above, in conjunction with equation~\eqref{eq:dans_pd} establishes the desired claim.
\qed

\section{Details of user study}\label{app:exp}
In this section, we provide the deferred details of the user study from Section~\ref{sec:pol_drive}.
\paragraph{Self-driving environment.} The self-driving environment consists of an autonomous car which can be controlled by providing real-valued inputs acceleration and angular acceleration at every time step. We allow the policies to have access to the dynamics of this environment. Observe that there is no explicit reward function in the environment and each policy differs in the way it optimizes a chosen reward function to drive the car forward in a safe manner.

\paragraph{Policies.} The MPC based Policies A-E were constructed by optimizing linear rewards comprising features F1-F9 as
\begin{itemize}
    \item[F1] Distance from the starting point along y-axis.
    \item[F2] Velocity of the autonomous car.
    \item[F3] Distance from the center of each lane.
    \item[F4] Gaussian collision detector for nearby objects.
    \item[F5] Collision detector which works at smaller radii than
    F4.
    \item[F6] Over-speeding feature which penalizes higher speeds.
    \item[F7] Reward for over-taking vehicles in the front.
    \item[F8] Gaussian off-road detector.
    \item[F9] Reward to promote speeding up near obstacles.
\end{itemize}
For each of the base policy, we set the weights of the features to encode different driving behaviors.
\begin{itemize}
    \item[Pol A]  programmed to prefer the right-most lane and progress forward at a slow speed.
    \item[Pol B] programmed to prefer the left-most lane and move forward as fast as possible.
    \item[Pol C] programmed to be conservative, avoids collision and proceeds forward.
    \item[Pol D] programmed to get attracted towards other cars and obstacles.
    \item[Pol E] programmed to prefer center lane and exhibit opportunistic behavior by moving ahead of other cars.
\end{itemize}

\paragraph{Details of target set and linear weights.} We selected the two data-oblivious sets to trade-off between the criteria C1-C5 as
\begin{equation}\label{eq:exp_targ}
\begin{small}
\begin{gathered}
  \Sc_1 = \{\z\; | \; \z \in [0,1]^{5}, \z_1 \geq 0.3, \z_2 \geq 0.3, \z_3 \geq 0.2, \z_4 \geq 0.3, \z_5 \geq 0.4 \},\\
  \Sc_2 = \{\z\; | \; \z \in [0,1]^{5}, \z_1 \geq 0.25, \z_2 \geq 0.25, \z_3 \geq 0.25, \z_4 \geq 0.25, \z_5 \geq 0.25, \z_1 + \z_5 \geq 0.9 \}.
\end{gathered}
\end{small}
\end{equation}
In addition, we selected 9 set of weights $\wght_{1:9}$ for linearly combining the different criteria.
\begin{itemize}
    \item[$\wght_1$:] Average of the users' self-reported weights.
    \item[$\wght_2$:] Weight vector obtained by regressing the overall criterion on C1-C5 with squared loss as \begin{equation*}
        \wght_2 \in \arg \min_{\wght \in \simplex_5} \sum_{\ir, \ic}(\prefov(\ir, \ic) - \sum_{\jj}\wght(j)\pref^\jj(\ir, \ic))^2.
    \end{equation*}
    \item[$\wght_3$:] Weight obtained by regressing Bradley-Terry-Luce (BTL) scores. The BTL parametric model assumes a real-valued score $v_i$ for each policy and posits that $\Pr(\text{Pol } i \prefeq \text{Pol }j) = \exp(v_i)/\exp(v_i) + \exp(v_j)$. Denoting the scores obtained from the overall preferences by $v^{\sf{ov}}$ and those obtained from the individual criteria by $v^\jj$ for $\jj \in [5]$, the weight
    \begin{equation*}
        \wght_2 \in \arg \min_{\wght \in \simplex_5} \sum_{\ii}
        (v^{\sf{ov}}_\ii - \sum_{\jj} \wght(\jj) v^\jj_i)^2.
    \end{equation*}
    \item[$\wght_4$:] Data-oblivious weight $\wght_4 =[0.2, 0.2, 0.2, 0.2, 0.2]$.
    \item[$\wght_5$:] Data-oblivious weight $\wght_5 =[0.25, 0.5/3, 0.5/3, 0.5/3, 0.25]$.
    \item[$\wght_6$:] Data-oblivious weight $\wght_6 =[0.30, 0.4/3, 0.4/3, 0.4/3, 0.30]$.
    \item[$\wght_7$:] Data-oblivious weight $\wght_7 =[0.5/3, 0.5/3, 0.25, 0.5/3, 0.25]$.
    \item[$\wght_8$:] Data-oblivious weight $\wght_8 =[0.4/3, 0.4/3, 0.3, 0.4/3, 0.30]$.
    \item[$\wght_9$:] Data-oblivious weight $\wght_9 =[0.3, 0.1/2, 0.3, 0.1/2, 0.3]$.
\end{itemize}
The set of data oblivious weights were chosen to account for different trade-offs along the criteria C1-C5 including the uniform weight $\wght_4$.

\paragraph{Data Collection.} Table~\ref{tab:phase_one} shows the comparison data collected from the Mturk users in both the phases of the experiment. The entry $i, j$ of the comparison matrices represents the fraction of users which preferred Policy $i$ over Policy $j$. The top 5 rows and columns of each matrix correspond to the baseline policies while the bottom rows correspond to the two randomized policies R1 and R2 obtained as the Blackwell winner corresponding to sets $\Sc_1$ and $\Sc_2$ respectively.

In addition, we would like to highlight some details from an
experiment design perspective. Since the experiment was run in two
phases, we could not guarantee the same set of subjects to participate
in both parts of the experiment. In order to limit distribution
shifts, we restricted the nationality of the subjects to United States
and began both the phases on the same time and day of the week. Also,
in order to prevent biased evaluations, the ordering of the policy
pairs as well as the ordering policies within a comparison was
randomized across the users.

Figures~\ref{fig:ins},~\ref{fig:pol} and~\ref{fig:ques} shows the
experiment setup we used for obtaining comparison data from Amazon
Mechanical Turk users consisting of the instructions, the policy
comparison page and the questionnaire that the users were asked to
fill out.  \setlength{\tabcolsep}{3pt}
\begin{table}[t]
  \begin{footnotesize}
\parbox{.33\linewidth}{
\centering
\begin{tabular}{c|ccccc}
  & A & B & C & D & E\\
  \hline
A & 0.50 & 0.64 & 0.45 & 0.41 & 0.39\\
B & 0.36 & 0.50 & 0.30 & 0.30 & 0.25\\
C & 0.55 & 0.70 & 0.50 & 0.55 & 0.57\\
D & 0.59 & 0.70 & 0.45 & 0.50 & 0.52\\
E & 0.61 & 0.75 & 0.43 & 0.48 & 0.50\\
\hline
R1 & 0.49 & 0.80 & 0.22 & 0.46 & 0.29\\
R2 & 0.49 & 0.88 & 0.66 & 0.61 & 0.41\\
\end{tabular}
\subcaption{C1: Aggressiveness}
}
\parbox{.33\linewidth}{
\centering
\begin{tabular}{c|ccccc}
  & A & B & C & D & E\\
  \hline
A & 0.50 & 0.57 & 0.50 & 0.50 & 0.41\\
B & 0.43 & 0.50 & 0.30 & 0.39 & 0.45\\
C & 0.50 & 0.70 & 0.50 & 0.43 & 0.59\\
D & 0.50 & 0.61 & 0.57 & 0.50 & 0.57\\
E & 0.59 & 0.55 & 0.41 & 0.43 & 0.50\\
\hline
R1 & 0.46 & 0.71 & 0.32 & 0.51 & 0.39\\
R2 & 0.51 & 0.71 & 0.61 & 0.59 & 0.51\\
\end{tabular}
\subcaption{C2: Predictability}
}
\parbox{.33\linewidth}{
\centering
\begin{tabular}{c|ccccc}
  & A & B & C & D & E\\
  \hline
A & 0.50 & 0.16 & 0.25 & 0.32 & 0.30\\
B & 0.84 & 0.50 & 0.89 & 0.82 & 0.68\\
C & 0.75 & 0.11 & 0.50 & 0.73 & 0.61\\
D & 0.68 & 0.18 & 0.27 & 0.50 & 0.41\\
E & 0.70 & 0.32 & 0.39 & 0.59 & 0.50\\
\hline
R1 & 0.73 & 0.22 & 0.76 & 0.78 & 0.76\\
R2 & 0.90 & 0.24 & 0.44 & 0.66 & 0.66\\
\end{tabular}
\subcaption{C3: Quickness}
}

\vspace{3mm}

\parbox{.33\linewidth}{
\centering
\begin{tabular}{c|ccccc}
  & A & B & C & D & E\\
  \hline
A & 0.50 & 0.59 & 0.45 & 0.57 & 0.39\\
B & 0.41 & 0.50 & 0.32 & 0.34 & 0.32\\
C & 0.55 & 0.68 & 0.50 & 0.48 & 0.59\\
D & 0.43 & 0.66 & 0.52 & 0.50 & 0.50\\
E & 0.61 & 0.68 & 0.41 & 0.50 & 0.50\\
\hline
R1 & 0.44 & 0.80 & 0.20 & 0.39 & 0.24\\
R2 & 0.41 & 0.80 & 0.71 & 0.59 & 0.39\\
\end{tabular}
\subcaption{C4: Conservativeness}
}
\parbox{.33\linewidth}{
\centering
\begin{tabular}{c|ccccc}
  & A & B & C & D & E\\
  \hline
A & 0.50 & 0.52 & 0.41 & 0.50 & 0.43\\
B & 0.48 & 0.50 & 0.32 & 0.55 & 0.55\\
C & 0.59 & 0.68 & 0.50 & 0.55 & 0.57\\
D & 0.50 & 0.45 & 0.45 & 0.50 & 0.50\\
E & 0.57 & 0.45 & 0.43 & 0.50 & 0.50\\
\hline
R1 & 0.54 & 0.68 & 0.32 & 0.49 & 0.41 \\
R2 & 0.63 & 0.73 & 0.59 & 0.61 & 0.54\\
\end{tabular}
\subcaption{C5: Collision Risk}
}
\parbox{.33\linewidth}{
\centering
\begin{tabular}{c|ccccc}
  & A & B & C & D & E\\
  \hline
A & 0.50 & 0.39 & 0.25 & 0.43 & 0.34\\
B & 0.61 & 0.50 & 0.30 & 0.50 & 0.50\\
C & 0.75 & 0.70 & 0.50 & 0.57 & 0.61\\
D & 0.57 & 0.50 & 0.43 & 0.50 & 0.48\\
E & 0.66 & 0.50 & 0.39 & 0.52 & 0.50\\
\hline
R1 & 0.66 & 0.76 & 0.29 & 0.59 & 0.39\\
R2 & 0.66 & 0.73 & \textbf{0.66} & 0.56 & 0.51\\
\end{tabular}
\subcaption{Overall Preferences}
}
\caption{Each matrix consists of pairwise comparisons between policies elicited from a user study with around 50 participants on Mturk. An entry $i, j$ of the comparison matrices represents the fraction of users which preferred Policy i over Policy j. Policies A-E comprise the base set of policies while Policies R1-R2 are the randomized Blackwell winners obtained from the sets in equation~\eqref{eq:exp_targ}. While Policy C is the overall von Neumann winner, Policy R2 is preferred over it by 66\% of the users.}
\label{tab:phase_one}
\end{footnotesize}
\end{table}

\begin{figure}
\captionsetup{font=small}
\includegraphics[scale = 0.4]{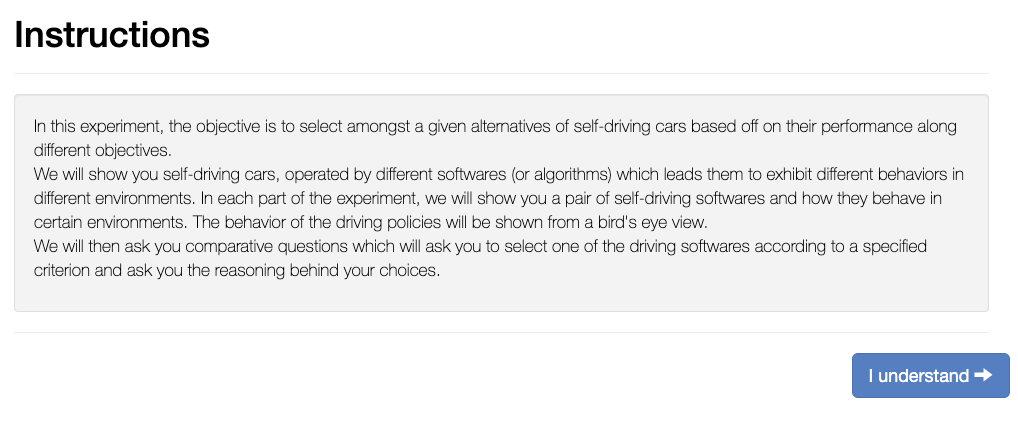}
\includegraphics[scale = 0.4]{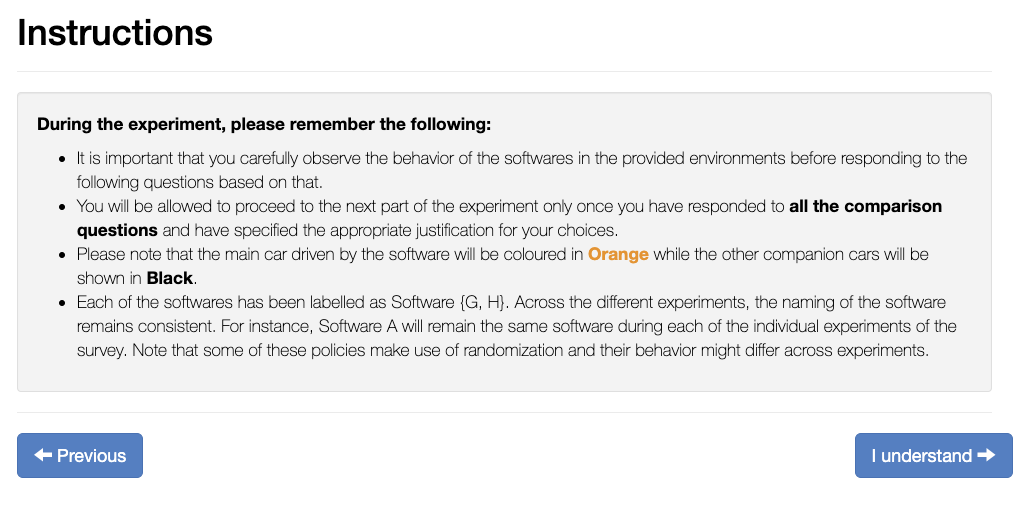}
\centering
\caption{Instructions provided to the users before the experiment began. The users were asked to compare behavior of policies and were told to expect some policies to exhibit a randomized behavior.}
\label{fig:ins}
\end{figure}

\begin{figure}
\includegraphics[scale = 0.4]{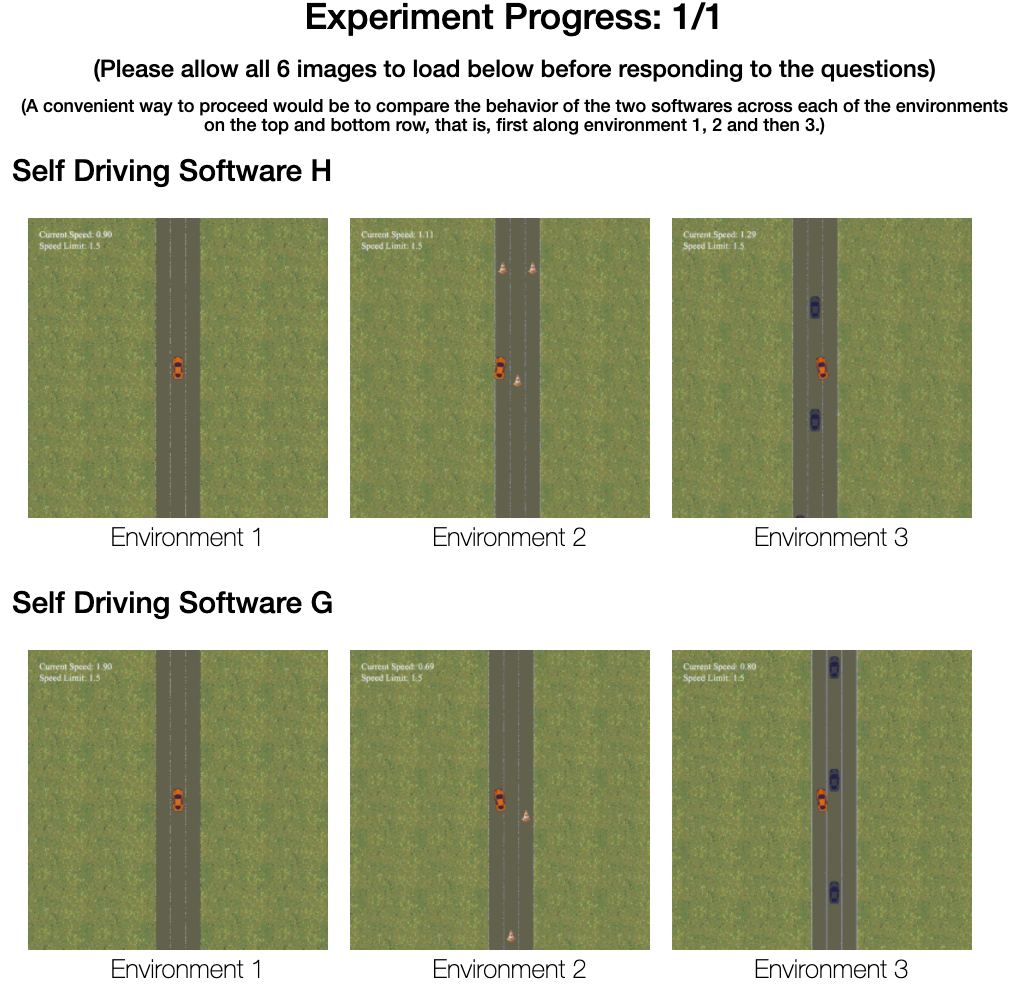}
\centering
\caption{Layout of the experiment where each panel shows a GIF exhibiting a Policy controlling the autonomous vehicle in one of the worlds of the environment. The users were instructed to compare behaviors across each of the columns before proceeding to answer the questions.}
\label{fig:pol}
\end{figure}

\begin{figure}
\includegraphics[scale = 0.4]{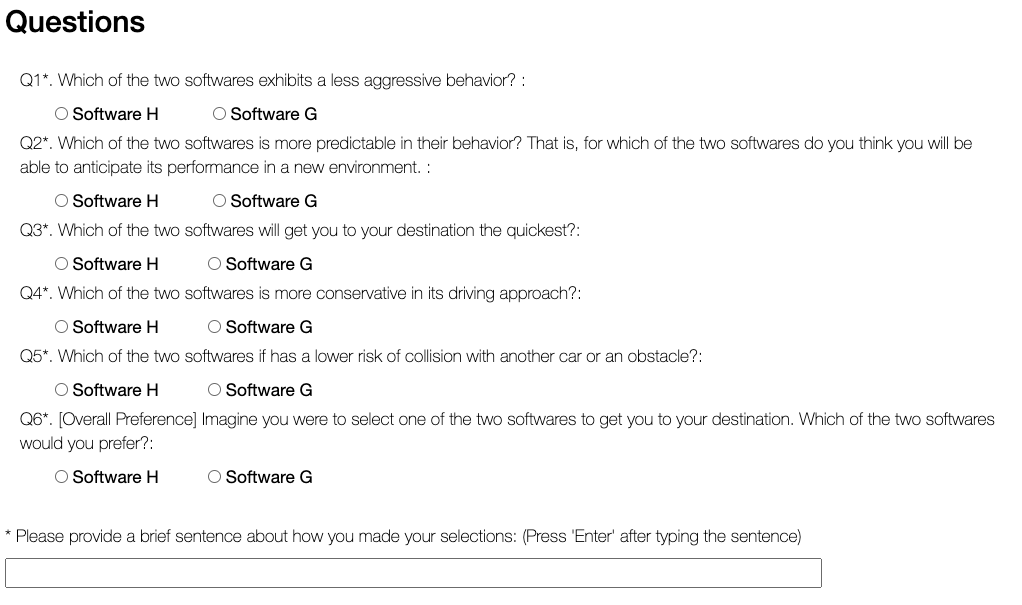}
\includegraphics[scale = 0.4]{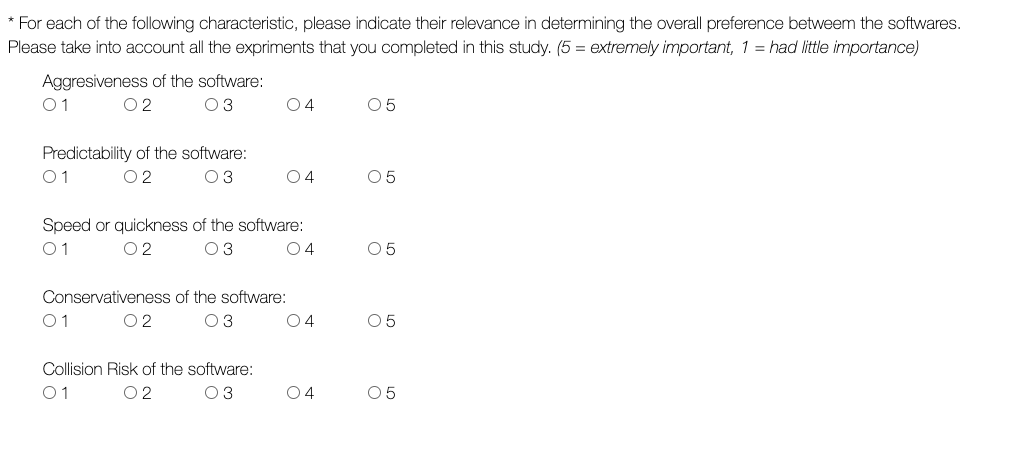}
\centering
\caption{Layout of the questions panel comprising the 6 comparison questions and the form for reporting the relevance of each criterion in the overall evaluation. }
\label{fig:ques}
\end{figure}

\newpage
\printbibliography
\end{document}